\documentclass[twoside]{article}

\usepackage[utf8]{inputenc}
\usepackage[english]{babel}
\usepackage{graphicx}
\usepackage{subcaption}
\usepackage{booktabs}
\usepackage{amsmath}
\usepackage{amsfonts}
\usepackage{tabularx}
\usepackage{amssymb}
\usepackage{amsthm}
\usepackage{color}
\usepackage{comment}
\usepackage{hyperref}

\usepackage[ruled]{algorithm2e}

\usepackage[accepted]{aistats2019}

\usepackage[round]{natbib}
\newtheorem{theorem}{Theorem}[section]
\newtheorem{corollary}{Corollary}[theorem]
\newtheorem{lemma}[theorem]{Lemma}

\newtheorem{prop}{Proposition}
\newenvironment{sketch}{%
  \proof}{\endproof}

\begin{document}

\twocolumn[

\aistatstitle{Differentiable Antithetic Sampling for Variance Reduction in Stochastic Variational Inference}

\aistatsauthor{Mike Wu \And Noah Goodman \And Stefano Ermon}

\aistatsaddress{Stanford University \And Stanford University \And Stanford University} ]

\begin{abstract}

Stochastic optimization techniques are standard in variational inference algorithms. These methods estimate gradients by approximating expectations with independent Monte Carlo samples.
In this paper, we explore a technique that uses correlated, but more \emph{representative}, samples to reduce variance.
Specifically, we show how to generate \textit{antithetic} samples with sample moments that match the population moments of an underlying proposal distribution.
Combining a \textit{differentiable} antithetic sampler with modern stochastic variational inference, we showcase the effectiveness of this approach for learning a deep generative model. An implementation is available at \url{https://github.com/mhw32/antithetic-vae-public}.
\end{abstract}

\section{Introduction}
\label{sec:introduction}

A wide class of problems in science and engineering can be solved by gradient-based optimization of function expectations.
This is especially prevalent in machine learning \citep{schulman2015gradient},  including variational inference \citep{ranganath2014black, rezende2014stochastic} and reinforcement learning \citep{silver2014deterministic}.
On the face of it, problems of this nature require solving an intractable integral.
Most practical approaches instead use Monte Carlo estimates of expectations and their gradients.
These techniques are unbiased but can suffer from high variance when sample size is small---one unlikely sample in the tail of a distribution can heavily skew the final estimate.
A simple way to reduce variance is to increase the number of samples; however the computational cost grows quickly. We would like to reap the positive benefits of a larger sample size using as few samples as possible. \textit{With a fixed computational budget, how do we choose samples?}

A large body of work has been dedicated to reducing variance in sampling, with the most popular in machine learning being  reparameterizations for some continuous distributions \citep{kingma2013auto,jang2016categorical} and control variates to adjust for estimated error \citep{mnih2014neural,weaver2001optimal}.
These techniques sample i.i.d.~but perhaps it is possible to choose correlated samples that are more \textit{representative} of their underlying distribution?
Several such non-independent sampling approaches have been proposed in statistics.
In this work we investigate \textit{antithetics}, where for every sample we draw, we include a negatively correlated sample to minimize the distance between sample and population moments.

The key challenges in applying antithetic sampling to modern machine learning are (1) ensuring that antithetic samples are correctly distributed such that they provide unbiased estimators for Monte Carlo simulation, and (2) ensuring that sampling is differentiable to permit gradient-based optimization.
We focus on stochastic variational inference and explore using antithetics for learning the parameters for a deep generative model. 
Critically, our method of antithetic sampling is differentiable and can be composed with reparametrizations of the underlying distributions to provide a fully differentiable sampling process. This yields a simple and low variance way to optimize the parameters of the variational posterior.

Concisely, our contributions are as follows:

\begin{itemize}
    \item We review a method to to generate Gaussian variates with known sample moments, then apply it to antithetics, and generalize it to other families using deterministic transformations.
    \item We show that differentiating through the sampling computation improves variational inference.
    \item We show that training VAEs with antithetic samples improves learning across objectives, posterior families, and datasets.
\end{itemize}

\section{Background}
\label{sec:background}


\subsection{Variational Inference and Learning}

Consider a generative model that specifies a joint distribution $p_\theta(x,z)$ over a set of observed variables $x \in \mathbb{R}^m$ and stochastic variables $z \in \mathbb{R}^d$ parameterized by $\theta$. We are interested in the posterior distribution $p_\theta(z|x) = \frac{p_\theta(x|z)p(z)}{p(x)}$, which is intractable since $p(x) = \int_z p(x,z)dz$. Instead, we introduce a \textit{variational posterior}, $q_\phi(z|x)$ that approximates $p_\theta(z|x)$ but is easy to sample from and to evaluate.

Our objective is to maximize the likelihood of the data (the ``evidence"), $\log p_\theta(x)$. This is intractable so we optimize the evidence lower bound (ELBO) instead:
\begin{align}
    \log p_\theta(x) & \geq \mathbb{E}_{q_\phi(z|x)}[\log \frac{p_\theta(x,z)}{q_\phi(z|x)}]\label{eqn:elbo}
\end{align}
The VAE \citep{kingma2013auto,rezende2014stochastic} is an example of one such generative model where $p_\theta(x|z)$ and $q_\phi(z|x)$ are both deep neural networks used to parameterize a simple likelihood (e.g., Bernoulli or Gaussian).

\paragraph{Stochastic Gradient Estimation}

Since $\phi$ can impact the ELBO (though not the true marginal likelihood it lower bounds), we jointly optimize over  $\theta$ and $\phi$.
The gradients of the ELBO objective are:
\begin{align}
    \nabla_\theta \textup{ELBO} &= \mathbb{E}_{q_\phi(z|x)}[\nabla_\theta \log p_\theta(x,z)] \label{eqn:elbo_gradient_theta} \\
    \nabla_\phi \textup{ELBO} &= \nabla_\phi \mathbb{E}_{q_\phi(z|x)}[\log \frac{p_\theta(x,z)}{q_\phi(z|x)}]
    \label{eqn:elbo_gradient_phi}
\end{align}
Eqn.~\ref{eqn:elbo_gradient_theta} can be directly estimated using Monte Carlo techniques. However, as it stands, Eqn.~\ref{eqn:elbo_gradient_phi} is difficult to approximate as we cannot distribute the gradient inside the expectation. Luckily, if we constrain $q_\phi(z|x)$ to certain families, we can reparameterize.






\paragraph{Reparameterization Estimators}

Reparameterization refers to isolating sampling from the gradient computation graph \citep{kingma2013auto,rezende2014stochastic}. If we can sample $z \sim q_\phi(z|x)$ by applying a deterministic function $z = g_\phi(\epsilon): \mathbb{R}^{d} \rightarrow \mathbb{R}^{d}$ to sampling from an unparametrized distribution, $\epsilon \sim R$, then we can rewrite Eqn.~\ref{eqn:elbo_gradient_phi} as:
\begin{equation}
    \nabla_\phi \textup{ELBO} = \mathbb{E}_{\epsilon \in R}[\nabla_z \log\frac{p_\theta(x,z(\epsilon))}{q_\phi(z(\epsilon)|x)} \nabla_\phi g_\phi(\epsilon)]
\label{eqn:gradient_reparam}
\end{equation}

which can now be estimated in the usual manner. As an example, if $q_\phi(z|x)$ is a Gaussian, $\mathcal{N}(\mu, \sigma^2)$ and we choose $R$ to be $\mathcal{N}(0, 1)$, then $g(\epsilon) = \epsilon * \sigma + \mu$.

\subsection{Antithetic Sampling}
Normally, we sample i.i.d.~from $q_\phi(z|x)$ and $R$ to approximate Eqns.~\ref{eqn:elbo_gradient_theta} and \ref{eqn:gradient_reparam}, respectively. However, drawing correlated samples could reduce variance in our estimation. Suppose we are given $k$ samples $z_1, z_2, ..., z_{k} \sim q_\phi(z|x)$. We could choose a second set of samples $z_{k+1}, z_{k+2},  ..., z_{2k} \sim q_\phi(z|x, z_1, ..., z_k)$ such that $z_{i+k}$ is somehow the ``opposite" of $z_{i}$. Then, we can write down a new estimator using both sample sets. For example, Eqn.~\ref{eqn:elbo_gradient_theta} can be approximated by:
\begin{equation}
\frac{1}{2k}\sum_{i=1}^{k} \nabla_\theta \log p_\theta(x, z_i) + \nabla_\theta \log p_\theta(x, z_{i+k})
\label{eqn:opt_antithetic}
\end{equation}
Assuming $z_{k+1}, ..., z_{2k}$ is marginally distributed according to $q_\phi(z|x)$, Eqn.~\ref{eqn:opt_antithetic} is unbiased. Moreover, if $q_\phi(z|x)$ is near symmetric, the variance of this new estimator will be cut significantly. But \textit{what does ``opposite" mean?} One idea is to define ``opposite" as choosing $z_{k+i}$ such that the moments of the combined sample set $z_1, ..., z_{2k}$ match the moments of $q_\phi(z|x)$. Intuitively, if $z_i$ is too large, then choosing $z_{k+i}$ to be too small can help rebalance the sample mean, reducing first order errors. Similarly, if our first set of samples is too condensed at the mode, then choosing antithetic samples with higher spread can stabilize the variance closer to its expectation. However, sampling $z_{k+1}, ..., z_{2k}$ with particular sample statistics in mind is a difficult challenge. To solve this, we first narrow our scope to Gaussian distributions, and later extend to other distribution families.





\section{Generating Gaussian Variates with Given Sample Mean and Variance}
\label{sec:methods}

\begin{figure}[tb]
    \centering
    \includegraphics[width=0.55\columnwidth]{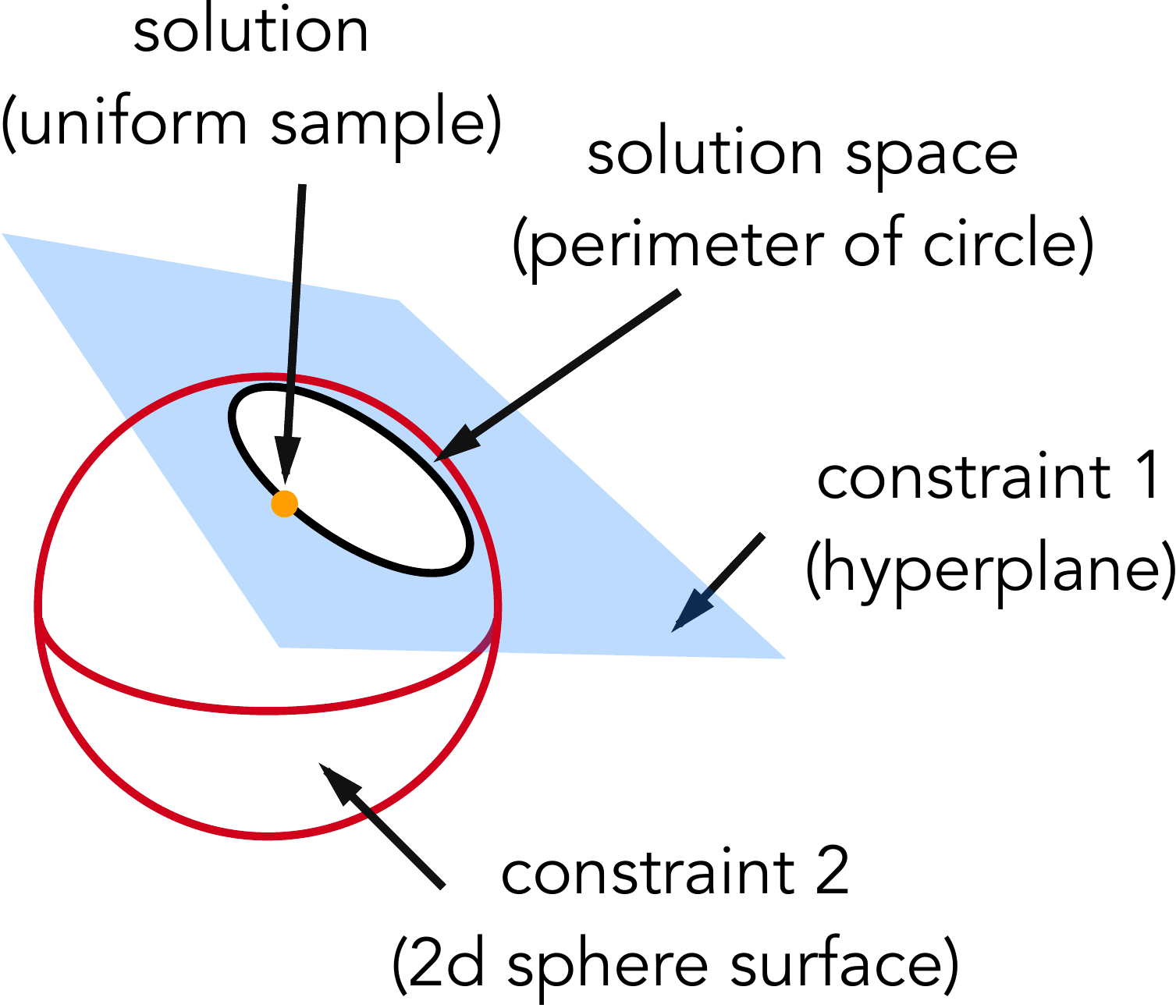}
    \caption{An illustration of Marsaglia's solution to the constrained sampling problem in two dimensions: build a $(k-1)$-dimensional sphere by intersecting a hyperplane and a $k$-dimensional sphere (each representing a contraint). Generating $k$ samples is equivalent to uniformly sampling from the perimeter of the circle.}
    \label{fig:algo}
\end{figure}



We present the \textit{constrained sampling problem}: given a Gaussian distribution with \emph{population mean} $\mu$ and \emph{population variance} $\sigma^2$, we wish to generate $k$ samples $x_1, ..., x_k \sim \mathcal{N}(\mu, \sigma^2)$ subject to the conditions:
\begin{align}
    \frac{1}{k}\sum_{i=1}^{k} x_i &= \eta  \label{eqn:constraint1}\\
    \frac{1}{k}\sum_{i=1}^{k} (x_i - \eta)^2 &= \delta^2 \label{eqn:constraint2}
\end{align}
where the constants $\eta$ and $\delta^2$ are given and represent the \emph{sample mean} and \emph{sample variance}. In other words, how can we draw samples from the correct marginal distribution conditioned on matching desired \emph{sample} moments? For example, we might wish to match \emph{sample} and \emph{population} moments: $\eta=\mu$ and $\delta=\sigma$.

Over forty years, there have been a handful of solutions. We review the algorithm introduced by \citep{marsaglia1980c69}. In our experiments, we reference a second algorithm by \citep{pullin1979generation, cheng1984generation}, which is detailed in the supplement. We chose \citep{marsaglia1980c69} due its simplicity, low computational overhead, and the fact that it makes the fewest random choices of proposed solutions.

\paragraph{Intuition} Since $x_1, ..., x_k$ are independent, we can write the joint density function as follows:
\begin{equation}
    p(x_1, ..., x_k) = (2\pi \sigma^2)^{-\frac{k}{2}}e^{-\frac{1}{2\sigma^2}\sum_i(x_i - \mu)^2}
\label{eqn:joint_density}
\end{equation}

We can interpret Eqn.~\ref{eqn:constraint1} as a hyperplane and Eqn.~\ref{eqn:constraint2} as the surface of a sphere in $k$ dimensions. Let $\mathcal{X}$ be the set of all points $(x_1, ..., x_k) \in \mathbb{R}^k$ that satisfy the above constraints. Geometrically, we can view $\mathcal{X}$ as the intersection between the hyperplane and $k$-dimensional sphere, i.e., the surface of a $(k-1)$ dimensional sphere (e.g. a circle if $k=2$).

We make the following important observation: the joint density (Eqn.~\ref{eqn:joint_density}) is constant for all points in $\mathcal{X}$. To see this, we can write the following:
\begin{equation*}
\begin{split}
    \sum_{i}(x_i - \mu)^2 &= \sum_i(x_i - \eta)^2 \\
    & \quad + 2(\eta - \mu)\sum_i(x_i - \eta) + k(\eta - \mu)^2 \\
    &= \sum_i(x_i - \eta)^2 + k(\eta - \mu)^2 \\
    &= k \delta^2 + k(\eta - \mu)^2
\end{split}
\label{eqn:proof1}
\end{equation*}
where $\sum_i(x_i - \eta) = \sum_i(x_i) - k\eta = 0$ by Eqn.~\ref{eqn:constraint1}. Plugging this into the density function, rewrite Eqn.~\ref{eqn:joint_density} as:
\begin{equation}
    p(x_1, ..., x_k) = (2\pi \sigma^2)^{-\frac{k}{2}}e^{-\frac{1}{2\sigma^2}(k \delta^2 + k(\eta - \mu)^2)}
\label{eqn:joint_pdf_2}
\end{equation}
Critically, Eqn.~\ref{eqn:joint_pdf_2} is independent of $x_1, ..., x_k$. For any $(\eta, \delta, \mu, \sigma)$, the density for every $x \in \mathcal{X}$ is constant. In other words, the conditional distribution of $x_1, ..., x_k$ given that $x_1, ..., x_k \in \mathcal{X}$ is the uniform distribution over $\mathcal{X}$. Surprisingly, it does \textit{not} depend on $\mu$ or $\sigma$.

Therefore, to solve the constrained sampling problem, we need only be able to sample uniformly from the surface of a $(k-1)$ dimensional sphere.

\paragraph{Marsaglia's Solution}
More precisely, we can generate the required samples $\textbf{x} = (x_1, ..., x_k)$ from a point $\textbf{z} = (z_1, ..., z_{k-1})$ uniformly distributed on the unit sphere in $\mathbb{R}^{k-1}$ centered at the origin by solving the linear system:
\begin{equation}
    \textbf{x} = k^{\frac{1}{2}}\delta\textbf{z}B + \eta \textbf{v}
\label{eqn:linear_system}
\end{equation}
where $\textbf{v} = (1, 1, ..., 1)$ is a $k$ dimensional vector of ones and $B$ is a $(k-1)$ by $k$ matrix such that the rows of $B$ form an orthonormal basis with the null space of $\textbf{v}$ i.e. we choose $B$ where $BB^{t} = I$ and $B\textbf{v}^{t} = 0$, which happens to satisfy our constraints:
\begin{align}
    \textbf{x}\textbf{v}^t &= k\eta \label{eqn:newc1}\\
    (\textbf{x} - \eta\textbf{v})(\textbf{x} - \eta\textbf{v})^t = k\delta^2\textbf{z}BB^t\textbf{z}^t &= k\delta^2 \label{eqn:newc2}
\end{align}
As $\textbf{z}$ is uniformly distributed over the unit $(k-1)$ sphere, Eqn.~\ref{eqn:newc1} and \ref{eqn:newc2} guarantee that $\textbf{x}$ is uniformly distributed in $\mathcal{X}$. We can generate $\textbf{z}$ by sampling $(\epsilon_1, ..., \epsilon_{k-1}) \sim \mathcal{N}(0, 1)$ and setting $z_i = \epsilon_i / \sum_i \epsilon_i^2$. As in \citep{marsaglia1980c69}, we set $B$ to \textsc{RowNormalize}$(A)$ where $A$ is defined as

\begin{center}
\tiny
$\begin{bmatrix}
1-k & 1 & 1 & . & . & . & 1 & 1& 1\\
0 & 2-k & 1 & . & . & . & 1 &1 & 1\\
0 &  0 & 3-k & . & . & . & 1 & 1 & 1\\
. &  &  &  &  &  &  & & .\\
 .&  &  &  &  &  &  & & .\\
. &  &  &  &  &  &  & &  .\\
0 & 0 & 0  & . & . & . &-2 & 1 & 1\\
0 & 0  & 0 & . & . &. & 0 &  -1 & 1
\end{bmatrix}$
\end{center}

and \textsc{RowNormalize} divides each row vector in $A$ by the sum of the elements in that row. We summarize the procedure in Alg.~\ref{algo:marsaglia} and the properties of \textsc{MarsagliaSample} in Prop.~\ref{prop:marsaglia}.

\begin{algorithm}[h!]
\SetAlgoLined
\caption{\textsc{MarsagliaSample}}
\KwData{i.i.d. samples $\epsilon_1, ..., \epsilon_{k-1} \sim \mathcal{N}(0, 1)$; Desired sample mean $\eta$ and variance $\delta^2$; Number of samples $k \in \mathbb{N}$.} 
\KwResult{A set of $k$ samples $x_1, x_2, ..., x_k$ marginally distributed as $\mathcal{N}(\mu, \sigma^2)$ with sample mean $\eta$ and sample variance $\delta^2$.}

$\gamma = \sqrt{n}\delta$\;
$s = \sum_i \epsilon_i^2$\;
\For{$i\gets 1$ \KwTo $k$}{
    $z_i = \epsilon_i[(k-i)(k-i+1)s]^{-1/2}$\;
}
$x_1 = (1-k)\gamma z_1 + \eta$\;
$x_k = \gamma \sum_{i=1}^{k-1} z_i + \eta$\;
\For{$i\gets 2$ \KwTo $k-1$}{
    $x_i = (\sum_{i=1}^{i-1} + (i-k)z_i)\gamma + \eta$\;
}
Return $x_1, ..., x_k$\;
\label{algo:marsaglia}
\end{algorithm}


\begin{prop}
    For any $k>2$, $\mu \in \mathbb{R}$ and $\sigma^2>0$, if $\eta \sim \mathcal{N}(\mu, \frac{\sigma^2}{k})$ and $\frac{(k-1)\delta^2}{\sigma^2} \sim \chi^2_{k-1}$ and $\boldsymbol{\epsilon} = \epsilon_1, ..., \epsilon_{k-1} \sim \mathcal{N}(0, 1)$ i.i.d., then the generated samples $x_1, ..., x_k = \textsc{MarsagliaSample}(\boldsymbol{\epsilon}, \eta, \delta^2,k)$ are independent normal variates sampled from $\mathcal{N}(\mu, \sigma^2)$ such that $\frac{1}{k}\sum_i x_i = \eta$ and $\frac{1}{k}\sum_i(x_i - \eta)^2 = \delta^2$.
    \label{prop:marsaglia}
\end{prop}
\begin{sketch}We provide a full proof in the supplement. For a sketch, let $\mathbf{x} = (x_1, ..., x_k)$ such that $x_i \sim \mathcal{N}(\mu, \sigma^2)$ i.i.d. Compute sample statistics $\eta, \delta^2$ from $\mathbf{x}$ as defined in Eqn.~\ref{eqn:sample_stats}. Consider the joint distribution over samples and sample moments:
\begin{equation*}
    p(\mathbf{x},\eta,\delta^2) = p(\eta,\delta^2)p(\mathbf{x}|\eta,\delta^2)
\end{equation*}
. We make two observations: first, $\eta,\delta^2$, as defined, are drawn from $p(\eta,\delta^2)$. Second, as hinted above, $p(\mathbf{x}|\eta,\delta^2)$ is the uniform distribution over a $(k-1)$-sphere, which Marsaglia shows us how to sample from. Thus, any samples $\mathbf{x'} \sim p(\mathbf{x}|\eta=\eta,\delta^2=\delta^2)$ will be distributed as $\mathbf{x}$ is (marginally), in other words i.i.d. Gaussian.
\end{sketch}


As implied in Prop.~\ref{prop:marsaglia}, if we happen to know the population mean $\mu$ and variance $\sigma^2$ (as we do in variational inference), we could generate $k$ i.i.d. Gaussian variates by sampling $\eta \sim \mathcal{N}(\mu, \frac{\sigma^2}{k})$ and $\frac{(k-1)\delta^2}{\sigma^2} \sim \chi^2_{k-1}$, and passing $\eta$, $\delta^2$ to \textsc{MarsagliaSample}.

\section{Constrained Antithetic Sampling}

We might be inclined to use \textsc{MarsagliaSample} to directly generate samples with some fixed \emph{deterministic} $\eta=\mu$ and $\delta=\sigma$. However, Prop.~\ref{prop:marsaglia} holds only if the desired sample moments $\eta,\delta$ are \emph{random} variables. If we choose them deterministically, we can no longer guarantee the correct marginal distribution for the samples, thus precluding their use for Monte Carlo estimates. Instead, what we can do is compute $\eta$ and $\delta^2$ from i.i.d. samples from $\mathcal{N}(\mu, \sigma^2)$, derive antithetic sample moments, and use \textsc{MarsagliaSample} to generate a second set of samples distributed accordingly.

More precisely, given a set of $k$ independent normal variates $(x_1, ..., x_k) \sim \mathcal{N}(\mu, \sigma^2)$, we would like to generate a new set of $k$ normal variates $(x_{k+1}, ..., x_{2k})$ such that the combined sample moments match the population moments, $\frac{1}{2k}\sum_{i=1}^{2k} x_i = \mu$ and $\frac{1}{2k}\sum_{i=1}^{2k} (x_i -\mu)^2 = \sigma^2$. We call the second set of samples $(x_{k+1}, ..., x_{2k})$ \textit{antithetic} to the first set.

We compute sample statistics from the first set:
\begin{align}
    \eta &= \frac{1}{k}\sum_{i=1}^{k} x_i & \delta^2 &= \frac{1}{k}\sum_{i=1}^{k} (x_i -\mu)^2
    \label{eqn:sample_stats}
\end{align}
Note that  $\eta,\delta$ are random variables, satisfying $\eta \sim \mathcal{N}(\mu, \frac{\sigma^2}{k})$ and $\frac{(k-1)\delta^2}{\sigma^2} \sim \chi^2_{k-1}$. Ideally, we would want the second set to come from an ``opposing" $\eta'$ and $\delta'$. To choose $\eta'$ and $\delta'$, we leverage the \textit{inverse CDF transform}: given the cumulative distribution function (CDF) for a random variable $X$, denoted $F_X$, we can define a uniform variate $Y= F_X(X)$. The \textit{antithetic} uniform variable is then $Y' = 1 - Y$, which upon application of the inverse CDF function, is mapped back to a properly distributed antithetic variate $X' = F^{-1}_X(Y')$. Crucially, $X$ and $X'$ have the same marginal distribution, but are not independent.

Let $F_{\eta}$ represent a Gaussian CDF and $F_{\delta}$ represent a Chi-squared CDF. We can derive $\eta'$ and $\delta'$ as:
\begin{align}
    \eta' &= F^{-1}_\eta(1 - F_\eta(\eta)) \\
    \frac{(k-1)(\delta')^2}{\sigma^2} &= F^{-1}_\delta\left(1 - F_\delta \left(\frac{(k-1)\delta^2}{\sigma^2}\right)\right) \label{eq:inversechi}
\end{align}
Crucially, $\eta', \delta'$ chosen this way are random variables with the correct marginal distributions, i.e., $\eta' \sim \mathcal{N}(\mu, \frac{\sigma^2}{k})$ and $\frac{(k-1)(\delta')^2}{\sigma^2} \sim \chi^2_{k-1}$. knowing $\eta', \delta'$, it is straightforward to generate antithetic samples with \textsc{MarsagliaSample}. We summarize the algorithm in Alg.~\ref{alg:antithetic} and its properties in Prop.~\ref{prop:antithetic}.

\begin{algorithm}[h!]
\SetAlgoLined
\caption{\textsc{AntitheticSample}}
\KwData{i.i.d. samples $(x_1, ..., x_k) \sim \mathcal{N}(\mu, \sigma^2)$; i.i.d. samples $\boldsymbol{\epsilon} = (\epsilon_1, ..., \epsilon_{k-1}) \sim \mathcal{N}(0, 1)$; Population mean $\mu$ and variance $\sigma^2$; Number of samples $k \in \mathbb{N}$.} 
\KwResult{A set of $k$ samples $(x_{k+1}, x_{k+2}, ..., x_{2k})$ marginally distributed as $\mathcal{N}(\mu, \sigma^2)$ with sample mean $\eta'$ and sample standard deviation $\delta'$.}

$v = k -1$\;
$\eta = \frac{1}{k}\sum_{i=1}^k x_i$\;
$\delta^2 = \frac{1}{k}\sum_{i=1}^k (x_i - \eta)^2$\;
$\eta' = F^{-1}_\eta(1 - F_\eta(\eta))$\;
$\lambda = v\delta^2/\sigma^2$\;
$\lambda' = F^{-1}_\delta(1 - F_\delta(\lambda))$\;
$(\delta')^2 = \lambda'\sigma^2/v$\;
$(x_{k+1}, ..., x_{2k}) = \textsc{MarsagliaSample}(\boldsymbol{\epsilon}, \eta', (\delta')^2, k)$\;
Return $(x_{k+1}, ..., x_{2k})$\;
\label{alg:antithetic}
\end{algorithm}

\begin{prop}
Given $k - 1$ i.i.d samples $\boldsymbol{\epsilon} = (\epsilon_1, ..., \epsilon_{k-1}) \sim \mathcal{N}(0, 1)$, $k$ i.i.d. samples $\boldsymbol{x} = (x_1, ..., x_k) \sim \mathcal{N}(\mu, \sigma^2)$, let $(x_{k+1}, ..., x_{2k}) = \textsc{AntitheticSample}(\boldsymbol{x}, \boldsymbol{\epsilon}, \mu, \sigma^2, k)$ be the generated antithetic samples. Then:
\begin{enumerate}
\item $x_{k+1}, ..., x_{2k}$ are independent normal variates sampled from $\mathcal{N}(\mu, \sigma^2)$.
\item The combined sample mean $\frac{1}{2k}\sum_{i=1}^{2k}x_i$ is equal to the population mean $\mu$.
\item The sample variance of $x_{k+1}, ..., x_{2k}$ is anticorrelated with the sample variance of $x_1, ..., x_k$.
\end{enumerate}
\label{prop:antithetic}
\end{prop}
\begin{proof}
The first property follows immediately from Prop.~\ref{prop:marsaglia}, as by construction $\eta', \delta'$ have the correct marginal distribution. Simple algebra shows that the inverse Gaussian CDF transform simplifies to $\eta' = 2 * \mu - \eta$, giving the desired relationship $\eta/2 + \eta'/2 = \mu $. The third property follows from Eq. \ref{eq:inversechi}.
\end{proof}

Since both sets of samples share the same (correct) marginal distribution, $x_1, ..., x_{2k}$ can be used to obtain unbiased Monte Carlo estimates.
\begin{prop}
        Given $k - 1$ i.i.d samples $\boldsymbol{\epsilon} = (\epsilon_1, ..., \epsilon_{k-1}) \sim \mathcal{N}(0, 1)$, $k$ i.i.d. samples $\boldsymbol{x} = (x_1, ..., x_k) \sim \mathcal{N}(\mu, \sigma^2)$, let $(x_{k+1}, ..., x_{2k}) = \textsc{AntitheticSample}(\boldsymbol{x}, \boldsymbol{\epsilon}, \mu, \sigma^2, k)$ be the generated antithetic samples. Then $\frac{1}{2k} \sum_{i=1}^{2k} f(x_i)$ is an unbiased estimator of $\mathbb{E}_{x \sim \mathcal{N}(\mu, \sigma^2)}[f(x)]$.
 \label{unbiased}
\end{prop}

\begin{proof}
Let $q(x_1, \cdots, x_k, x_{k+1}, \cdots, x_{2k})$ denote the joint distribution of the $2k$ samples. Note the two groups of samples $(x_1, ..., x_k)$ and $(x_{k+1}, ..., x_{2k})$ are not independent. However,
\begin{align*}
\mathbb{E}_{(x_1, \cdots, x_k, x_{k+1}, \cdots, x_{2k}) \sim q} \left[ \frac{1}{2k} \sum_{i=1}^{2k} f(x_i) \right] = \\
\frac{1}{2k} \sum_{i=1}^{2k} \mathbb{E}_{x_i \sim q_i(x_i)} \left[f(x_i) \right] = \mathbb{E}_{x \sim \mathcal{N}(\mu, \sigma^2)}[f(x)]
\end{align*}
because by assumption and Prop.~\ref{prop:antithetic}, each $x_i$ is marginally distributed as $\mathcal{N}(\mu, \sigma^2)$.
\label{prop:test}
\end{proof}

\subsection{Approximate Antithetic Sampling}

If $F_\eta$ and $F_\delta$ were well-defined and invertible, we could use Alg.~\ref{alg:antithetic} as is, with its good guarantees. On one hand, since $\eta$ is normally distributed, the inverse CDF transform simplifies to:
\begin{equation}
    \eta' = 2 * \mu - \eta
\label{eqn:inversecdf_mean}
\end{equation}
However, there is no general closed form expression for $F^{-1}_\delta$. Our options are then to either use a discretized table of probabilities or approximate the inverse CDF. Because we desire differentiability, we choose to use a normal approximation to $F_{\delta}$.

\paragraph{Antithetic Hawkins-Wixley}

\cite{canal2005normal} surveys a variety of normal approximations, all of which are a linear combination of $\chi^2$ variates to a power root. We choose to use \citep{hawkins1986note} as (P1) it is an even power, (P2) it contains only 1 term involving a random variate, and (P3) is shown to work better for smaller degrees of freedom (smaller sample sizes). We derive a closed form for computing $\delta'$ from $\delta$ by combining the normal approximation with Eqn.~\ref{eqn:inversecdf_mean}. We denote this final transform as the \textit{antithetic Hawkins-Wixley transform}:
\begin{equation}
    \lambda' = v(2(1 - \frac{3}{16v} - \frac{7}{512v^2} + \frac{231}{8192v^3}) - (\frac{\lambda}{v})^{1/4})^4
    \label{eqn:anti_approx}
\end{equation}
where $\lambda \sim \chi^2_{v}$ with $v$ being the degree of freedom. Therefore, if we set $\lambda = (k-1)\delta^2/\sigma^2 \sim \chi^2_{k-1}$ and $v = k - 1$, then we can derive $(\delta')^2 = \lambda'\sigma^2/(k-1)$ where $\lambda'$ is computed as in Eqn.~\ref{eqn:anti_approx}, whose derivation can be found in the supplementary material.

P1 is important as odd degree approximations e.g. \citep{wilson1931distribution} can result in a negative value for $\lambda'$ under small $k$. P2 is required to derive a closed form as most linear combinations do not factor. P3 is desirable for variational inference.

To update Alg.~\ref{alg:antithetic}, we swap the fourth line with Eqn.~\ref{eqn:inversecdf_mean} and the sixth line with Eqn.~\ref{eqn:anti_approx}.
The first property in Prop. \ref{prop:antithetic} and therefore also Prop. \ref{unbiased} do not hold anymore: the approximate \textsc{AntitheticSample} has bias that depends on the approximation error in Eqn.~\ref{eqn:anti_approx}.
%
In practice, we find the approximate \textsc{AntitheticSample} to be effective.
From now on, when we refer to \textsc{AntitheticSample}, we refer to the approximate version. See supplement for a written algorithm. We refer to Fig.~\ref{fig:demo} for an illustration of the impact of antithetics: sampling i.i.d. could result in skewed sample distributions that over-emphasize the mode or tails, especially when drawing very few samples. Including antithetic samples helps to ``stabilize" the sample distribution to be closer to the true distribution.

\begin{figure}[h!]
    \centering
    \begin{subfigure}[b]{\columnwidth}
        \includegraphics[width=\columnwidth]{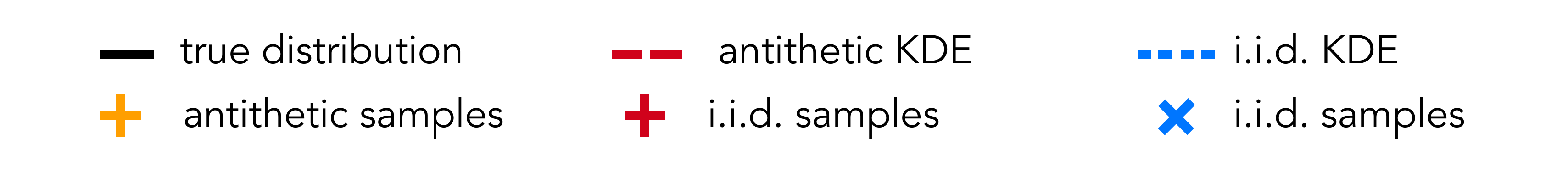}
    \end{subfigure}
    \begin{subfigure}[b]{.24\columnwidth}
        \includegraphics[width=\columnwidth]{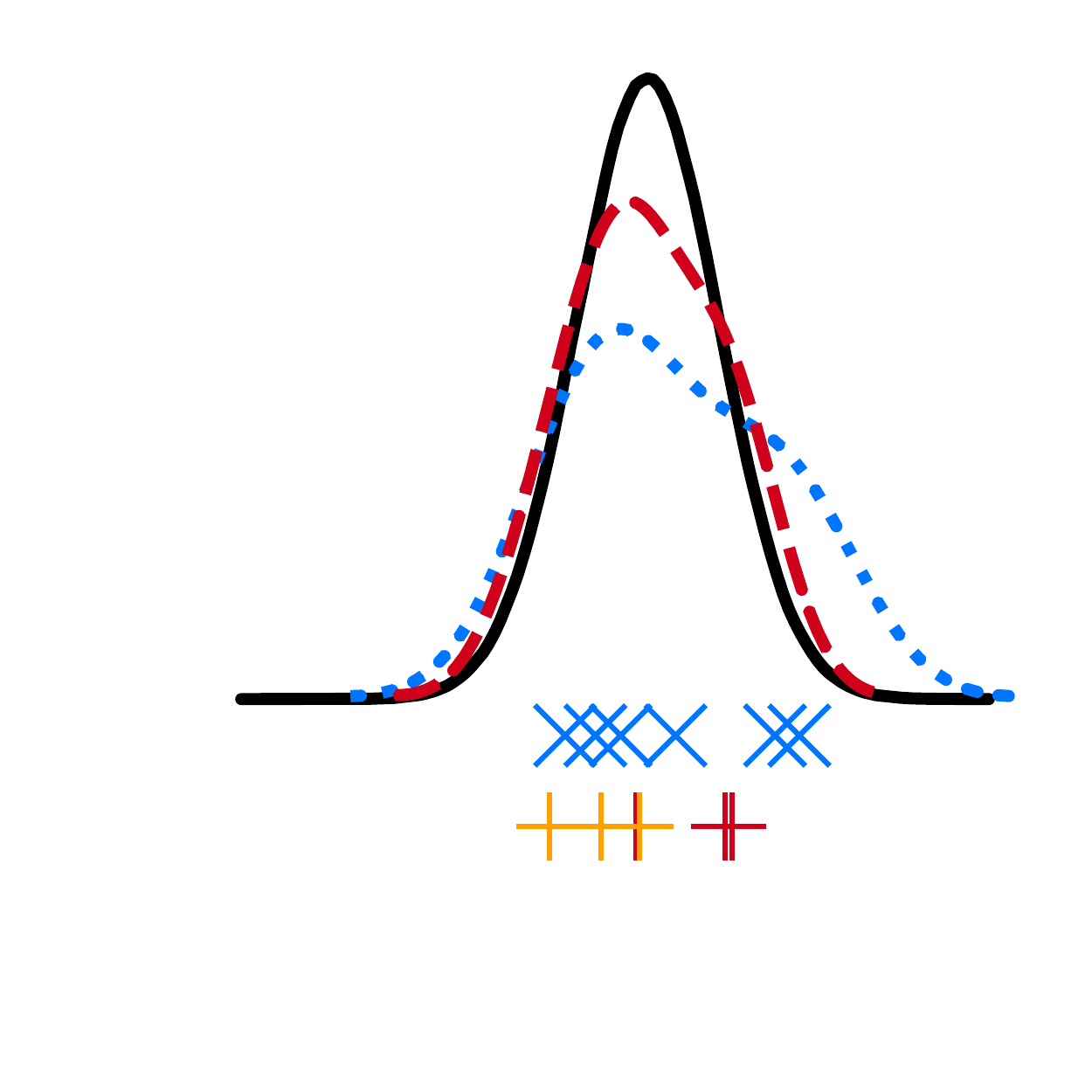}
        \caption{$k=8$}
        \label{fig:demo8}
    \end{subfigure}
    \begin{subfigure}[b]{.24\columnwidth}
        \includegraphics[width=\columnwidth]{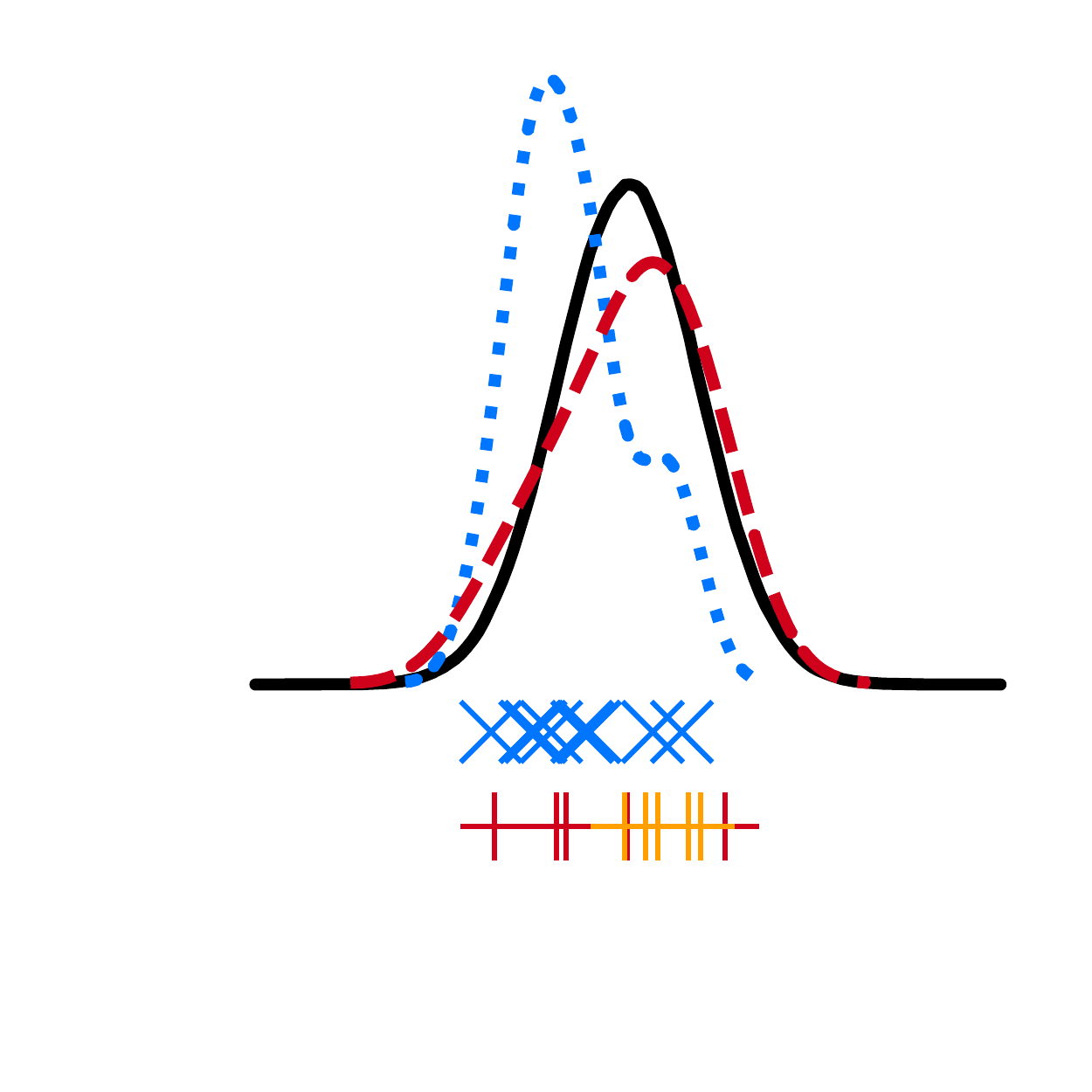}
        \caption{$k=10$}
        \label{fig:demo10}
    \end{subfigure}
    \begin{subfigure}[b]{.24\columnwidth}
        \includegraphics[width=\columnwidth]{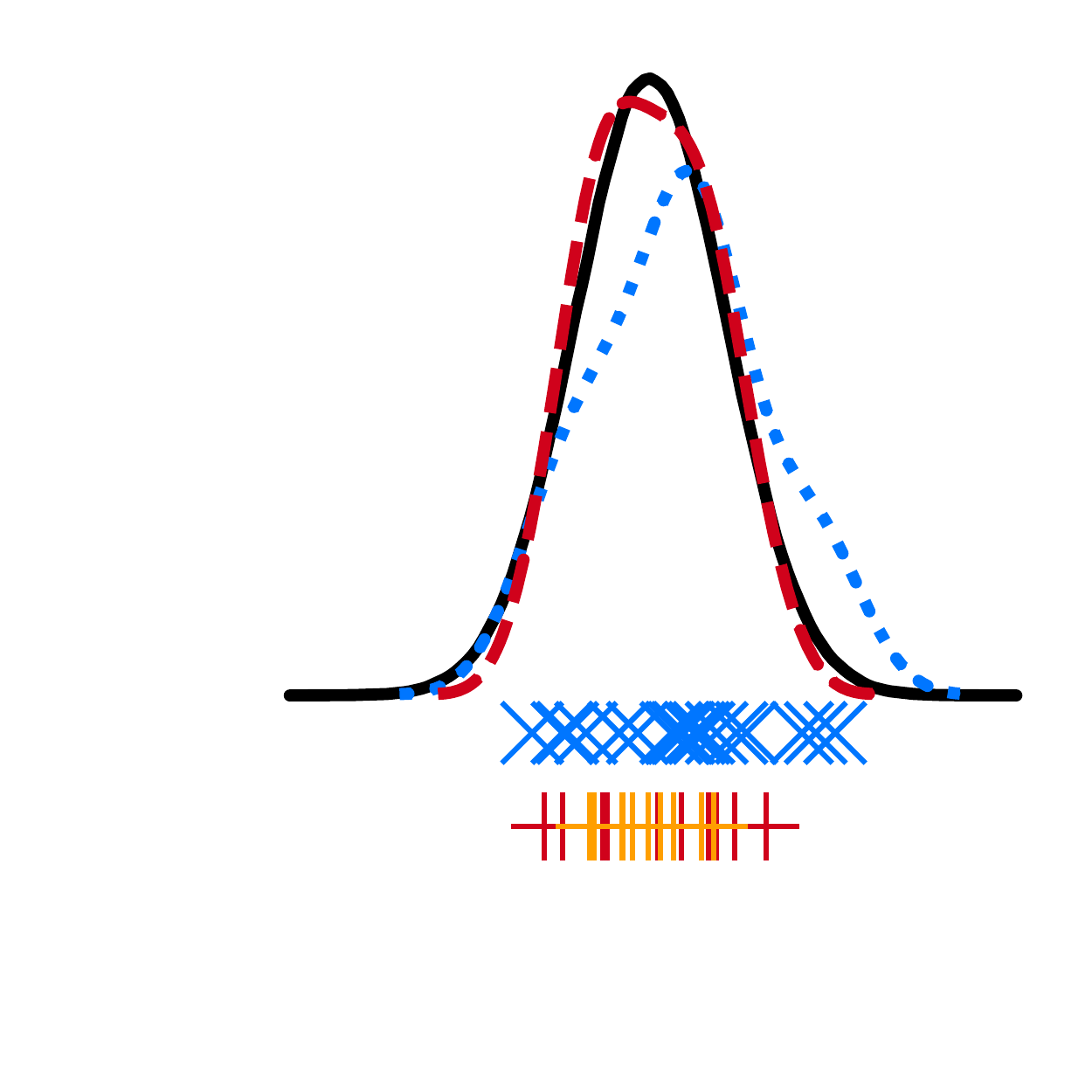}
        \caption{$k=20$}
        \label{fig:demo20}
    \end{subfigure}
    \begin{subfigure}[b]{.24\columnwidth}
        \includegraphics[width=\columnwidth]{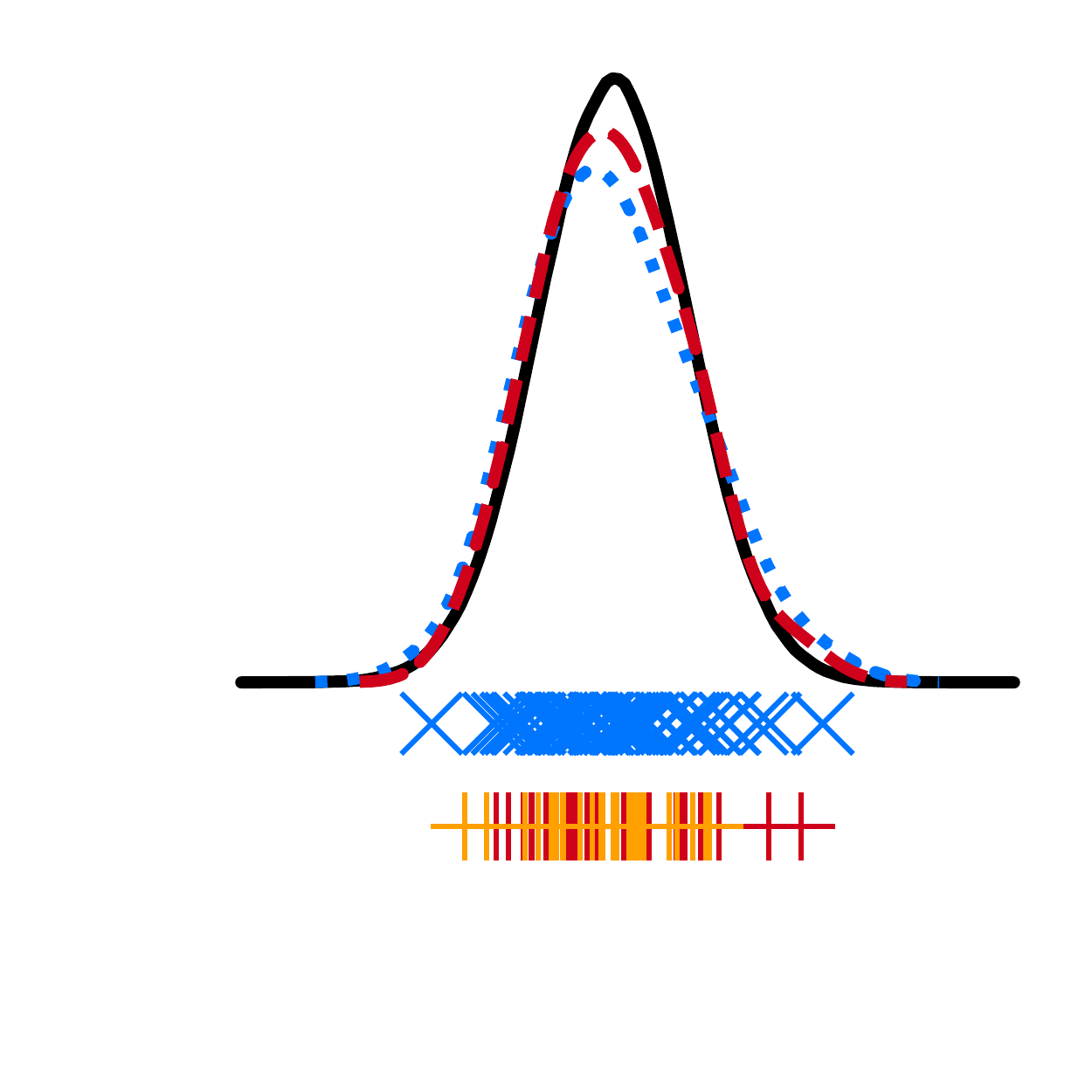}
        \caption{$k=50$}
        \label{fig:demo50}
    \end{subfigure}
    \caption{The effect of \textsc{AntitheticSample} in 1 dimension. We vary the number of samples $k$, and plot the true distribution (solid black line), a kernel density estimate (KDE) of the empirical distribution (dotted blue line) of $2k$ i.i.d. samples (blue), and a KDE of the empirical distribution (dashed red line) of $k$ i.i.d. samples (red) pooled with $k$ antithetic samples (orange).
    This snapshot was taken from the first epoch of training an AntiVAE on dynamic MNIST.}

    \label{fig:demo}
\end{figure}




\section{Generalization to Other Families}
\label{sec:generalization}


\cite{marsaglia1980c69}'s algorithm is restricted to distribution families that can be transformed to a unit sphere (primarily Gaussians), as are many similar algorithms \citep{cheng1984generation, pullin1979generation}. However, we can explore ``generalizing" \textsc{AntitheticSample} to a wider class of families by first antithetically sampling in a Gaussian distribution, then transforming its samples to samples from another family using a deterministic function, $g: \mathbb{R}^d \rightarrow \mathbb{R}^d$. Although we are not explicitly matching the moments of the derived distributions, we expect that transformations of more representative samples in an initial distribution may be more representative in the  transformed distribution. We now discuss a few candidates for $g(\cdot)$.

\subsection{One-Liners}
\cite{devroye1996random} presents a large suite of ``one line" transformations between distributions. We focus on three examples starting from a Gaussian to (1) Log Normal, (2) Exponential, and (3) Cauchy. Many additional transformations (e.g.~to Pareto, Gumbel, Weibull, etc.) can be used in a similar fashion. Let $F_x$ refer to the CDF of a random variable $x$. See supplementary material for derivations.


\paragraph{Log Normal}$g(z) = e^z$ where $z \sim \mathcal{N}(\mu, \sigma^2)$.



\paragraph{Exponential} Let $F_x(x) = 1 - \exp^{\lambda x}$ where $\lambda\in \mathbb{R}^d$ is a learnable parameter. Then $F_x^{-1}(y) = -\frac{1}{\lambda} \log y$. Thus, $g(u, \lambda) = -\frac{1}{\lambda} \log u$ where $u \in U(0, 1)$.

\paragraph{Cauchy} Let $F_x(x) = \frac{1}{2} + \frac{1}{\pi}\arctan(\frac{x - x_0}{\gamma})$ where $x_0\in \mathbb{R}^d, \gamma\in \mathbb{R}^d$ are learnable parameters. Then $F_x^{-1}(y) = \gamma(\tan(\pi y) + x_0)$. Given $u \in U(0, 1)$, we define $g(u, x_0, \gamma) = \gamma(\tan(\pi u) + x_0)$.

\subsection{Deeper Flows}
One liners are an example of a simple flow where we know how to score the transformed sample. If we want more flexible distributions, we can apply normalizing flows (NF). A \textit{normalizing flow} \citep{rezende2015variational} applies $T$ invertible transformations $h^{(t)}, t = 1, ..., T$ to samples $z^{(0)}$ from a simple distribution, leaving $z^{(T)}$ as a sample from a complex distribution. A common normalizing flow is a linear-time transformation: $g(z) = z + u(h(w^{T}z + b))$ where $w\in \mathbb{R}^d, u\in \mathbb{R}^d, b\in \mathbb{R}$ are learnable parameters, and $h$ is a non-linearity. In variational inference, flows enable us to parameterize a wider set of posterior families.

We can also achieve flexible posteriors using volume-preserving flows (VPF), of which \cite{tomczak2016improving} introduced the Householder transformation: $g(z) = (I - 2\frac{v \cdot v^T}{\|v\|^2})z$ where $v \in \mathbb{R}^d$ is a trainable parameter. Critically, the Jacobian-determinant is 1.
\begin{algorithm}[h]
\SetAlgoLined
\caption{AntiVAE Inference}
\KwData{A observation $x$; number of samples $k \geq 6$; a variational posterior $q(z|x)$ e.g. a $d$-dimensional Gaussian, $\mathcal{N}^d(\mu, \sigma^2)$.}
\KwResult{Samples $z^d_1,...,z^d_k\sim q_{\mu, \sigma}(z|x)$ that match moments.}
$\mu^d$, $\sigma^d =$\textsc{InferenceNetwork}$(x)$\;
$\mu = \normalfont\textsc{Flatten}(\mu^d)$\;
$\sigma = \normalfont\textsc{Flatten}(\sigma^d)$\;
$\epsilon_1, ..., \epsilon_{kd/2} \sim \mathcal{N}(0, 1)$\;
$\boldsymbol{\xi} = \xi_1, ..., \xi_{\frac{kd}{2} - 1} \sim \mathcal{N}(0, 1)$\;
\For{$i\gets 1$ \KwTo $kd/2$}{
    $y_i = \epsilon_i * \sigma + \mu$\;
}
$\boldsymbol{y} = (y_1, ..., y_{kd/2})$\;
$y_{\frac{kd}{2} + 1}, ..., y_{kd} = $\textsc{AntitheticSample}$(\boldsymbol{y}, \boldsymbol{\xi}, \mu, \sigma)$\;
$\boldsymbol{z} = (y_1, ..., y_{kd})$\;
$z^d_1, ..., z^d_k = \normalfont\textsc{UnFlatten}(\boldsymbol{z})$\;
Return $z^d_1, ..., z^d_k$\;
\label{algo:vae_forward}
\end{algorithm}

\section{Differentiable Antithetic Sampling}



Finally, we can use \textsc{AntitheticSample} to approximate the ELBO for variational inference.

For a given observation $x \in p_\textup{data}(x)$ from an empirical dataset, we write the antithetic gradient estimators as:
\begin{equation}
    \begin{split}
        \nabla_\theta \textup{ELBO} &\approx \frac{1}{2k}\sum_{i=1}^{k}[\nabla_\theta \log p_\theta(x, z_i) \\
        & + \nabla_\theta \log p_\theta(x, z_{i+k}) ]\label{eqn:grad_anti1}
    \end{split}
\end{equation}
\begin{equation}
    \begin{split}
        \nabla_\phi \textup{ELBO} &\approx \frac{1}{2k}\sum_{i=1}^{k}[\nabla_{z} \log \frac{p_\theta(x, z_i(\epsilon))}{q_\phi(z_i(\epsilon)|x)}\nabla_\phi g_\phi(\epsilon_i) \\
        & + \nabla_{z} \log \frac{p_\theta(x, z_{i+k}(\epsilon))}{q_\phi(z_{i+k}(\epsilon)|x)}\nabla_\phi g_\phi(\epsilon_{i+k})]\label{eqn:grad_anti2}
    \end{split}
\end{equation}
where $(\epsilon_1, ..., \epsilon_k) \sim \mathcal{N}(0, 1)$, $\boldsymbol{\xi} = (\xi_1, ..., \xi_{k-1})$, $\boldsymbol{z} = (z_1, ..., z_k) \sim q_\phi(z|x)$, and $(z_{k+1}, ..., z_{2k}) = \textsc{AntitheticSample}(\boldsymbol{z}, \boldsymbol{\xi}, \mu, \sigma^2, k)$. Optionally, $\boldsymbol{z} = \textsc{Transform}(\boldsymbol{z}, \alpha)$ where $\textsc{Transform}$ denotes any sample transformation(s) with parameters $\alpha$.

Alternative variational bounds have been considered recently, including an importance-weighted estimator of the ELBO, or IWAE \citep{burda2015importance}. Antithetic sampling can be applied in a similar fashion, as also shown in~\citep{shu2019buffered}.

Importantly, \textsc{AntitheticSample} is a special instance of a reparameterization estimator. Aside from samples from a parameter-less distribution (unit Gaussian), \textsc{AntitheticSample} is completely deterministic, meaning that it is differentiable with respect to the population moments $\mu$ and $\sigma^2$ by any modern auto-differentiation library.
Allowing backpropagation through \textsc{AntitheticSample} means that any free parameters are aware of the sampling strategy. Thus, including antithetics will change the optimization trajectory, resulting in a different variational posterior than if we had used i.i.d  samples alone. In Sec.~\ref{sec:results}, we show experimentally that \textit{most of the benefit} of differentiable antithetic sampling comes from being differentiable.

Alg.~\ref{algo:vae_forward} summarizes inference in a VAE using differentiable antithetic sampling (denoted by AntiVAE\footnote{For some experiments, we use Cheng's algorithm instead of Marsaglia's. We refer to this as AntiVAE (Cheng).}).
To the best of our knowledge, the application of antithetic sampling to stochastic optimization, especially variational inference is novel. Both the application of \citep{marsaglia1980c69} to drawing antithetics and the extension of \textsc{AntitheticSample} to other distribution families by transformation is novel. This is also the first instance of differentiating through an antithetic sample generator.
\begin{figure*}[t!]
    \centering
    \begin{subfigure}[b]{\textwidth}
        \includegraphics[width=\textwidth]{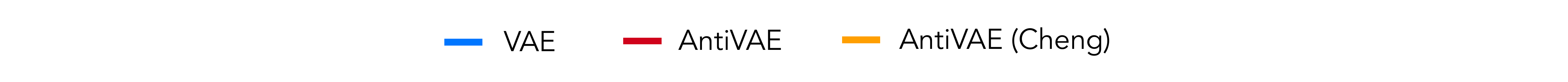}
    \end{subfigure}
    \begin{subfigure}[b]{0.135\textwidth}
        \includegraphics[width=\textwidth]{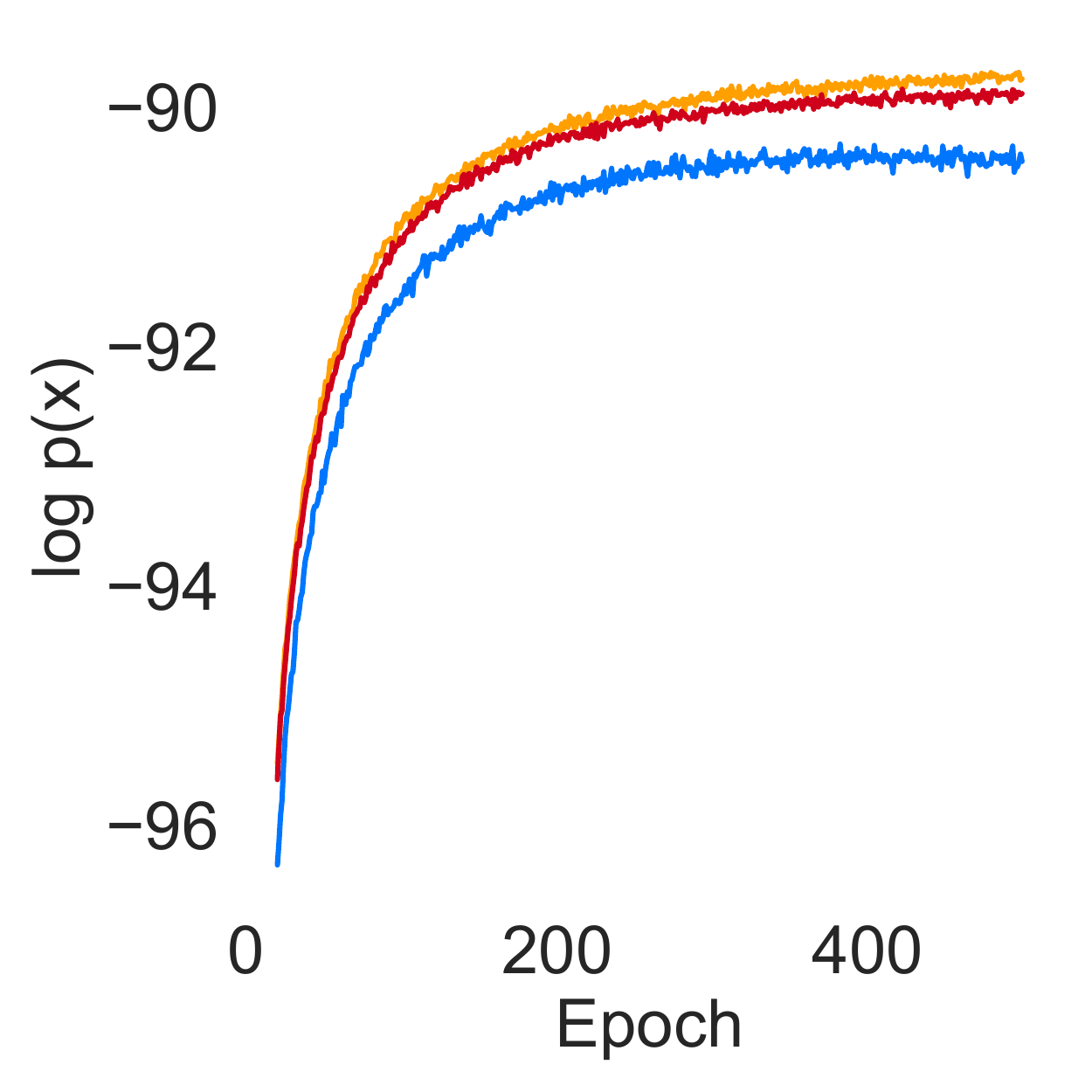}
        \caption{static}
        \label{fig:elbo_static_mnist}
    \end{subfigure}
    \begin{subfigure}[b]{0.135\textwidth}
        \includegraphics[width=\textwidth]{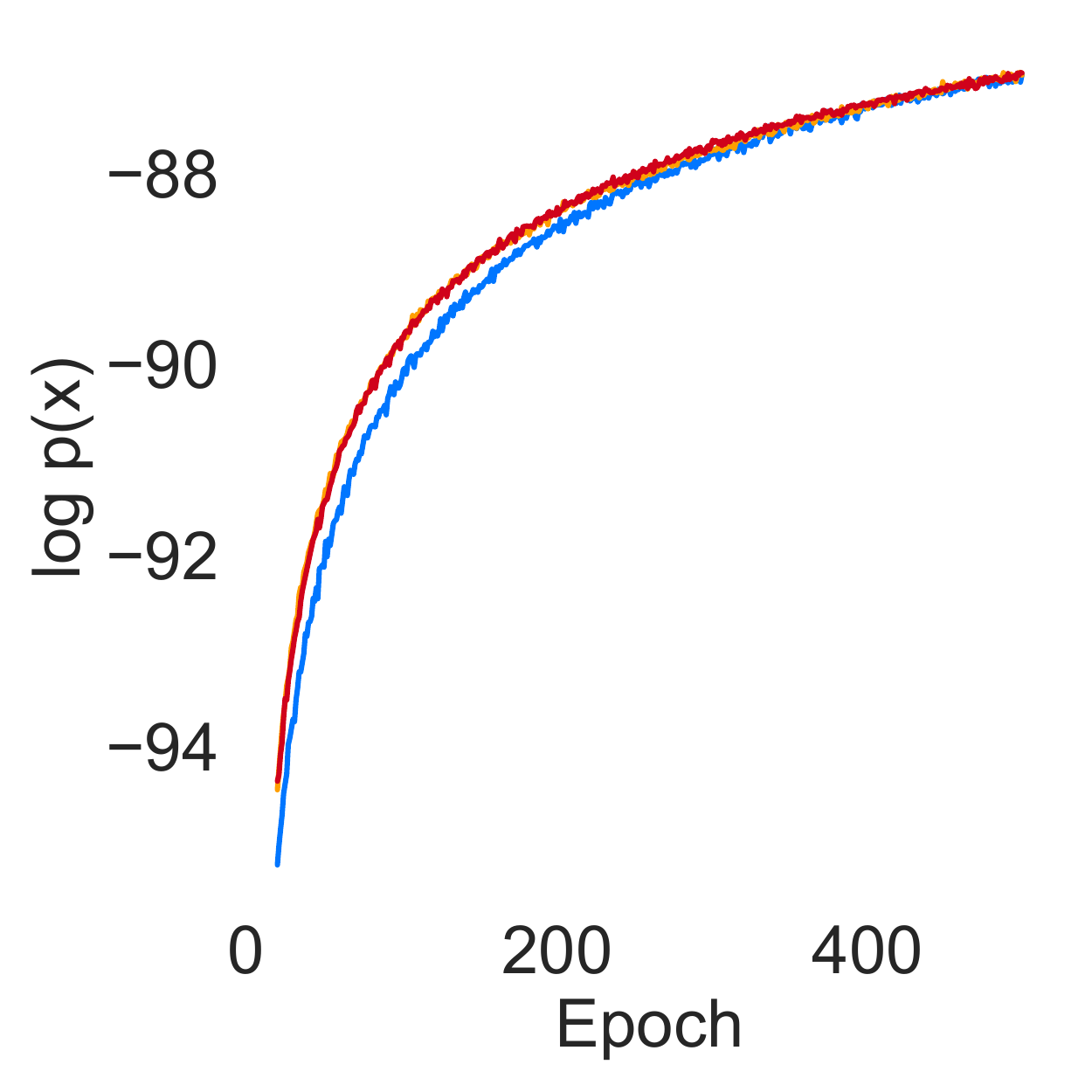}
        \caption{dynamic}
        \label{fig:elbo_dynamic_mnist}
    \end{subfigure}
    \begin{subfigure}[b]{0.135\textwidth}
        \includegraphics[width=\textwidth]{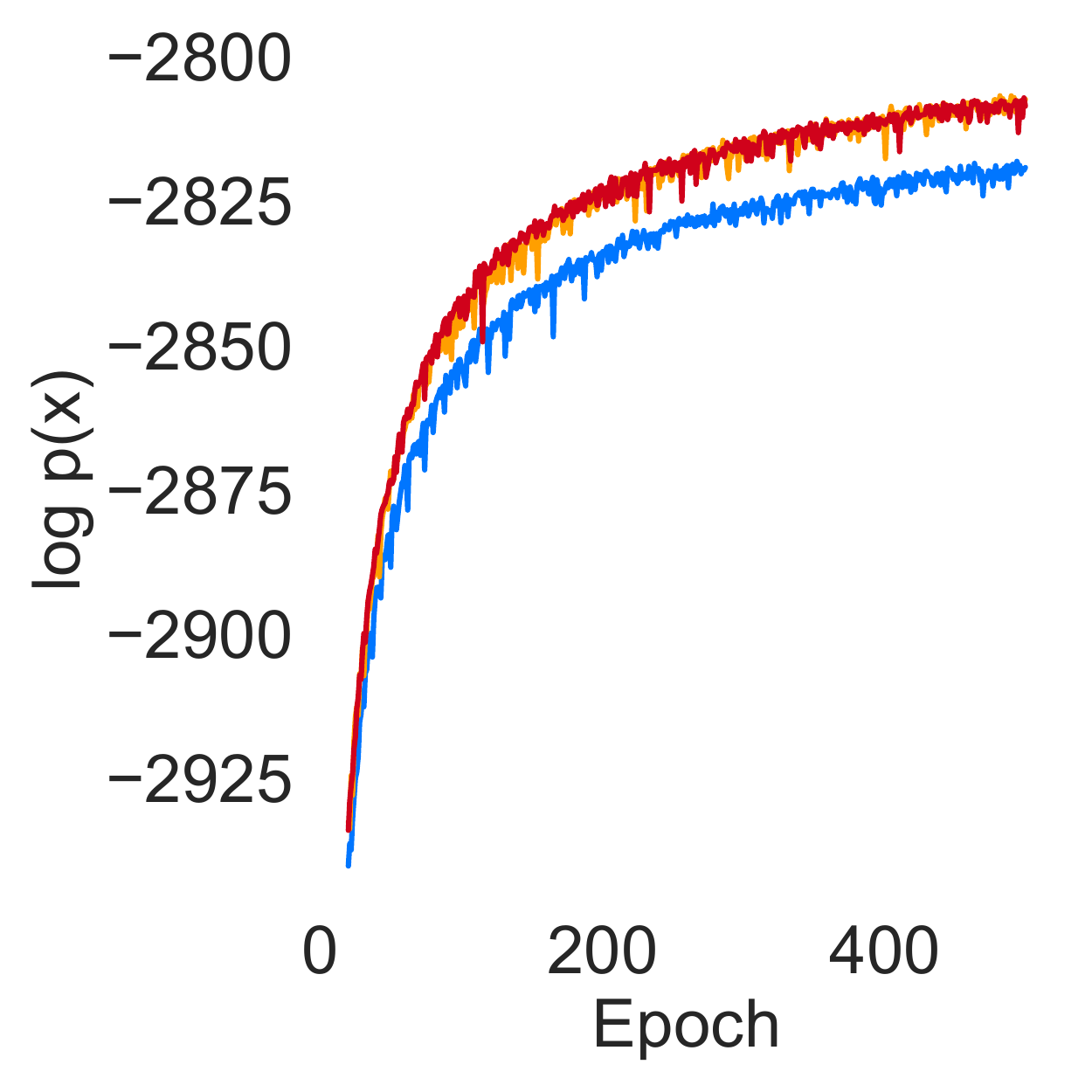}
        \caption{Fashion}
        \label{fig:elbo_fashion_mnist}
    \end{subfigure}
    \begin{subfigure}[b]{0.135\textwidth}
        \includegraphics[width=\textwidth]{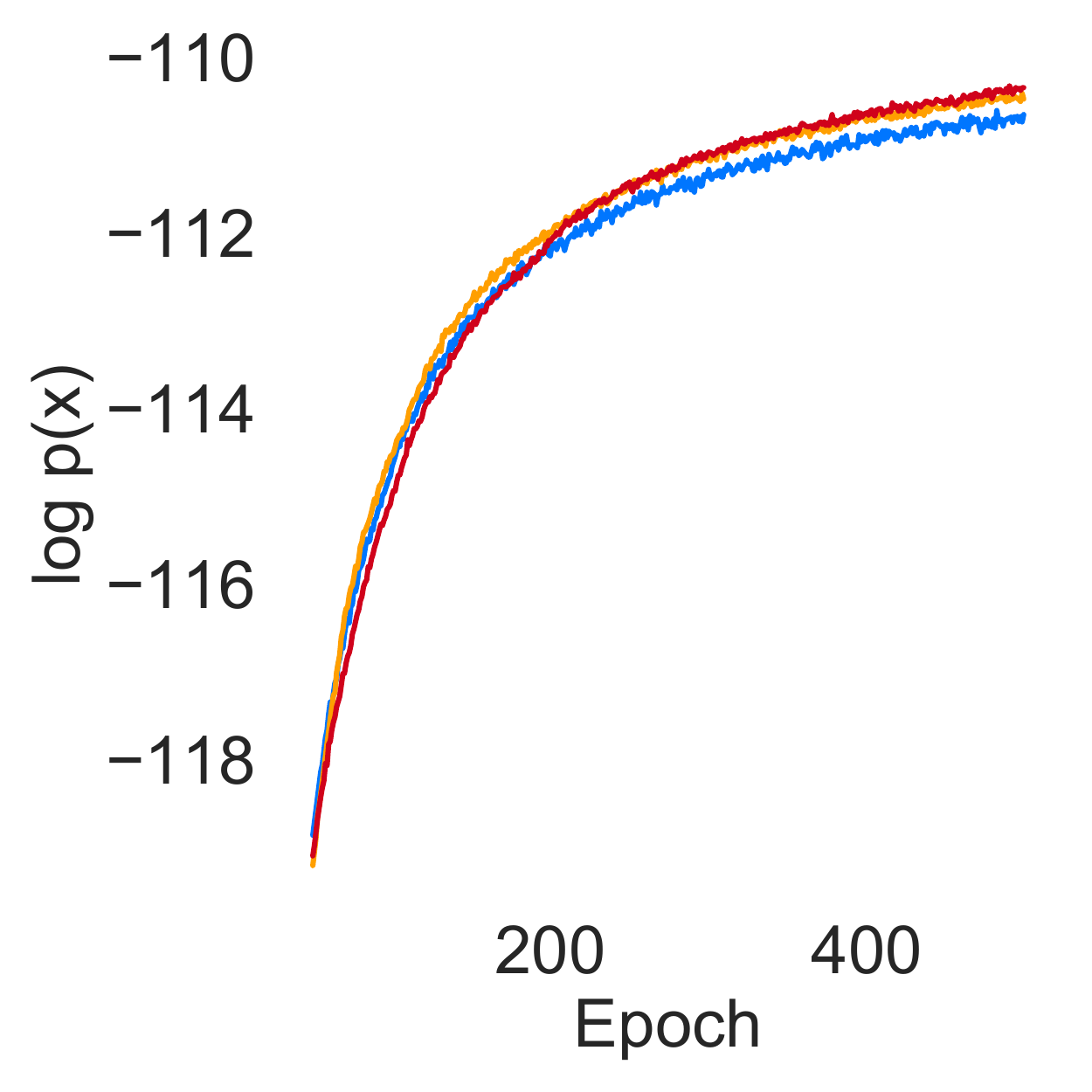}
        \caption{Omniglot}
        \label{fig:elbo_omniglot}
    \end{subfigure}
    \begin{subfigure}[b]{0.135\textwidth}
        \includegraphics[width=\textwidth]{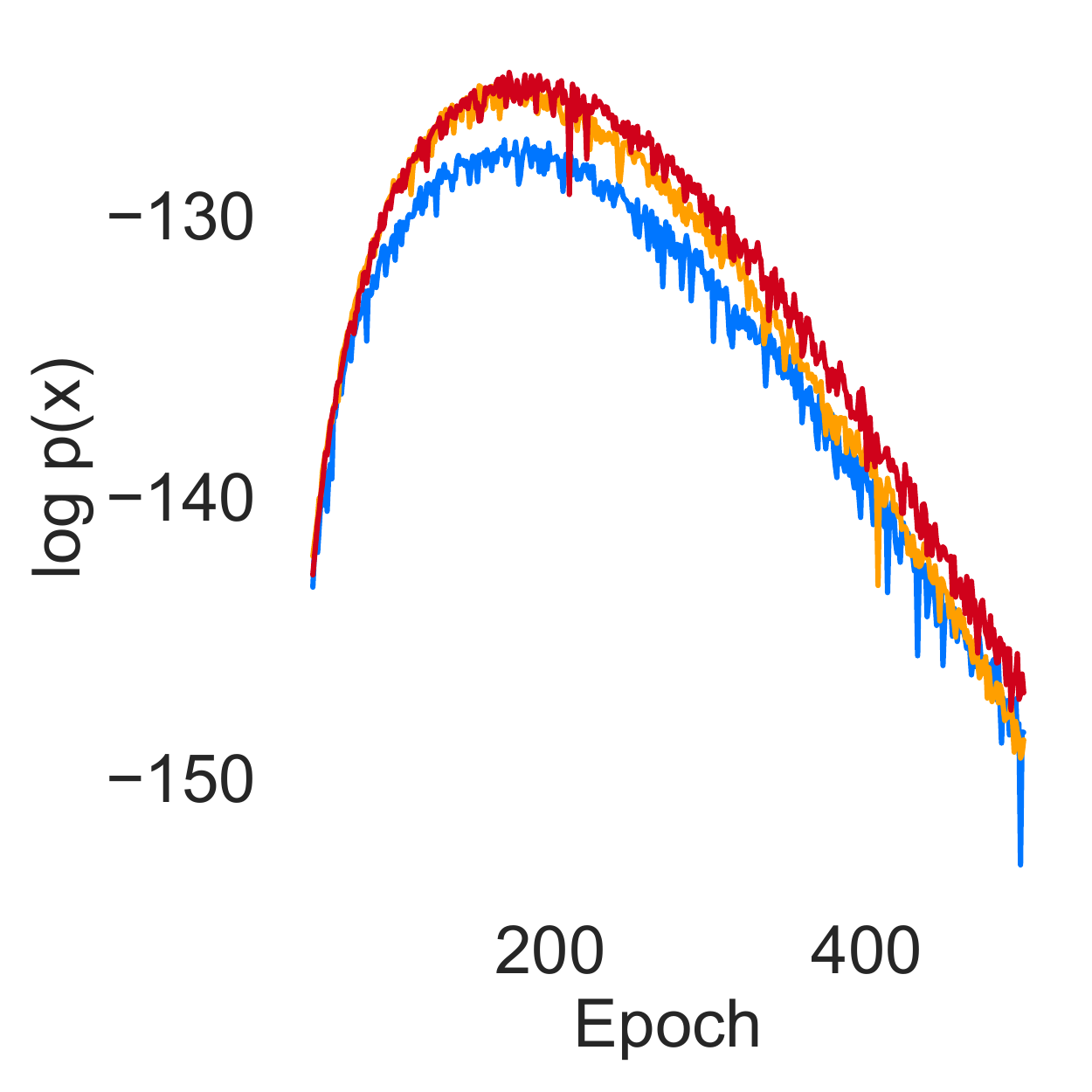}
        \caption{Caltech}
        \label{fig:elbo_caltech}
    \end{subfigure}
    \begin{subfigure}[b]{0.135\textwidth}
        \includegraphics[width=\textwidth]{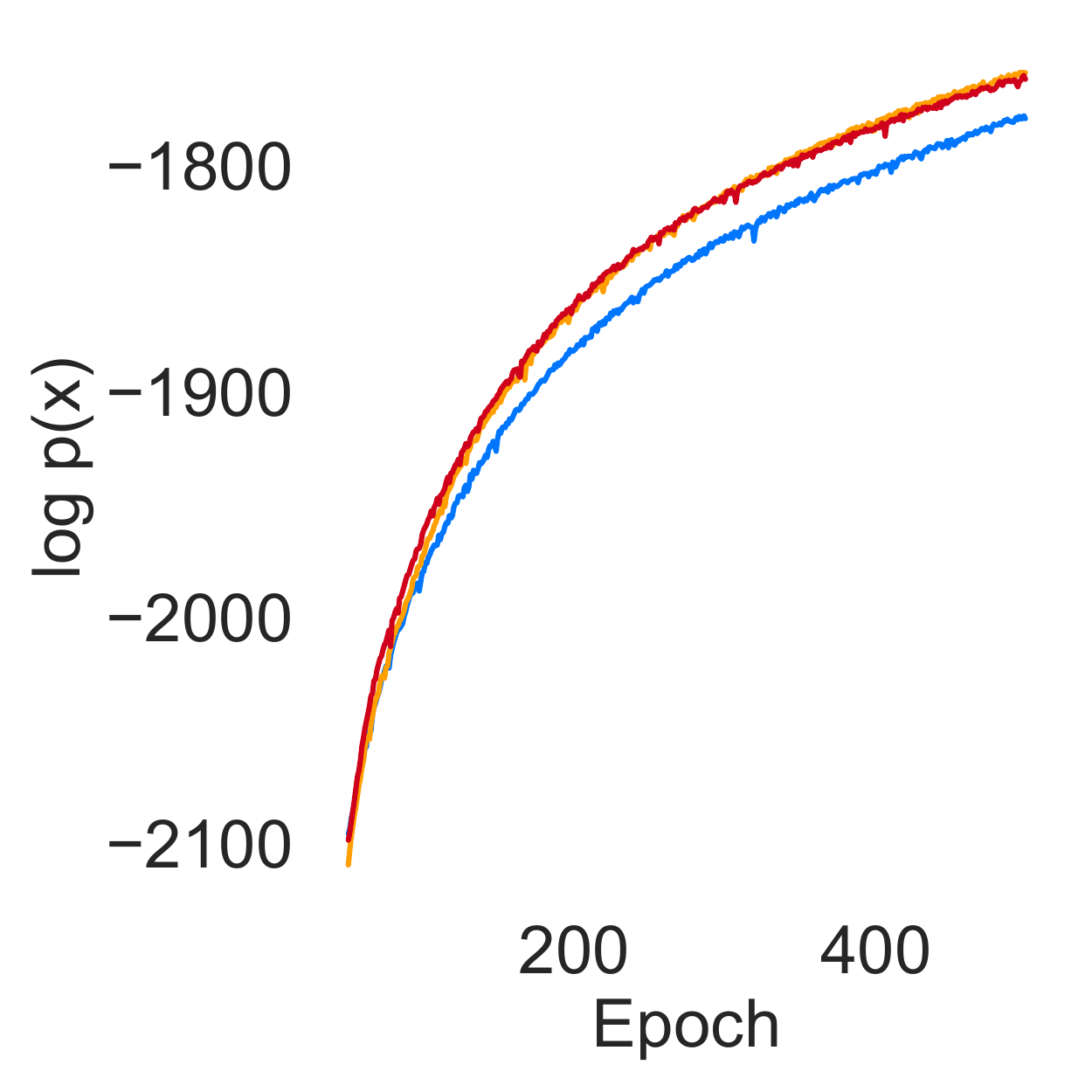}
        \caption{Frey}
        \label{fig:elbo_freyfaces}
    \end{subfigure}
    \begin{subfigure}[b]{0.135\textwidth}
        \includegraphics[width=\textwidth]{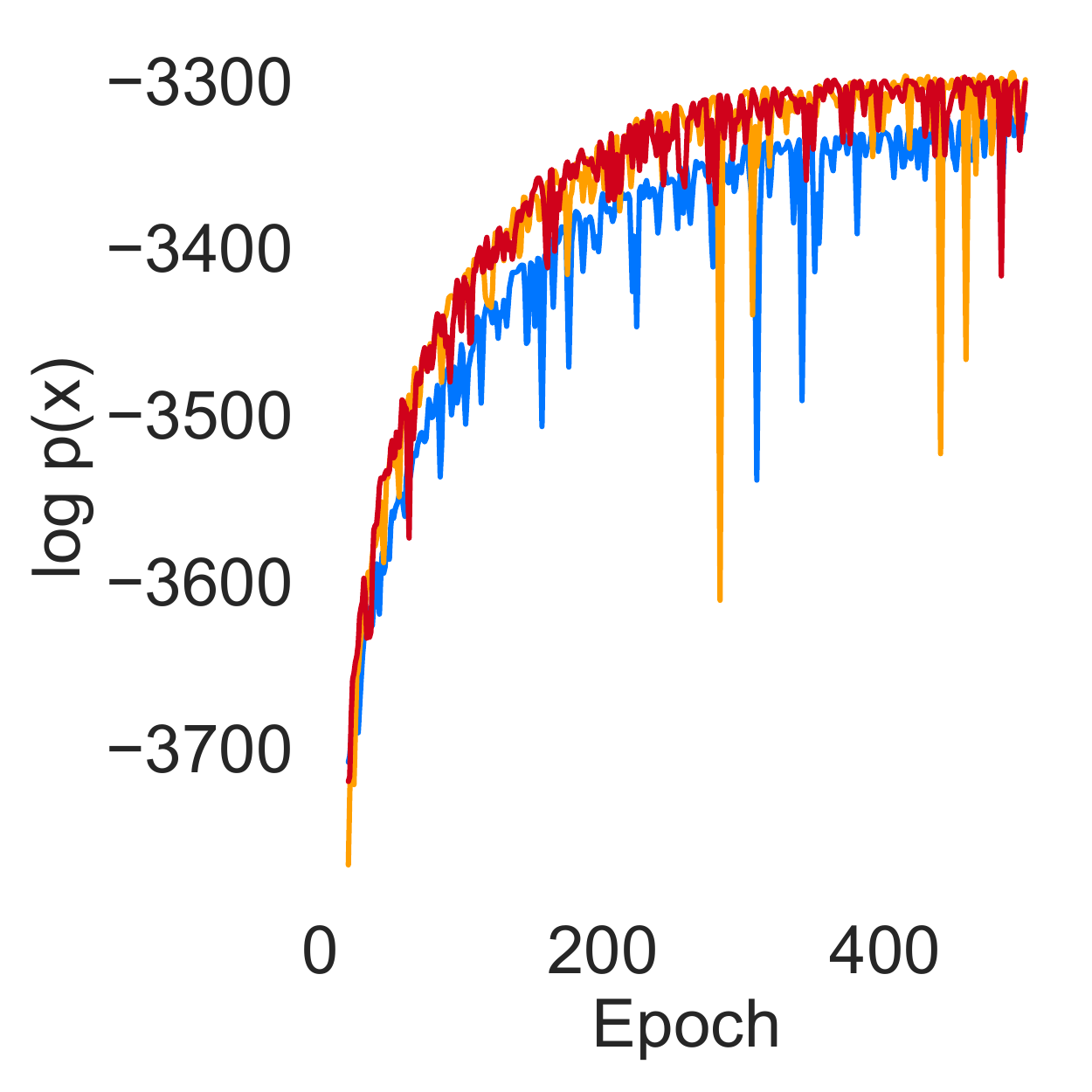}
        \caption{Hist.}
        \label{fig:elbo_histopathology}
    \end{subfigure}
\caption{A comparison of test log likelihoods over 500 epochs between VAE and AntiVAE. Transforming samples to match moments seems to have different degrees of effectiveness depending on the data domain. However, we find that the test ELBO with AntiVAE is almost always greater or equal to that of the VAE. This behavior is not sensitive to hyperparameters e.g. learning rate or MLP hidden dimension. For each subplot, we start plotting from epoch 20 to 500. We cannot resample observations in Caltech101, leading to overfitting.}
\label{fig:learning_trajectory}
\end{figure*}

\begin{table*}[t!]
\small
\begin{tabular}{r|ccccccc}
    Model & stat. MNIST & dyn. MNIST & FashionMNIST & Omniglot & Caltech  & Frey & Hist. \\
    \toprule
    VAE & -90.44 & -86.96 & -2819.13 & -110.65 & -127.26 & -1778.78 & -3320.37\\
    AntiVAE & -89.74 & -86.94 & -2807.06 & \textbf{-110.13} & \textbf{-124.87} & -1758.66 & -3293.01 \\
    AntiVAE (Cheng) & \textbf{-89.70} & \textbf{-86.93} & \textbf{-2806.71} & -110.39 &  -125.19 & \textbf{-1758.29} & \textbf{-3292.72}\\
    \hline
    VAE+IWAE & -89.78 & -86.71 & -2797.02 & \textbf{-109.32} & -123.99 & -1772.06 & -3311.23 \\
    AntiVAE+IWAE & \textbf{-89.71} & \textbf{-86.62} & \textbf{-2793.01} & -109.48 & \textbf{-123.35} & \textbf{-1771.47} & \textbf{-3305.91}\\
    \hline
    VAE ($\log\mathcal{N}$) & \textbf{-149.47} & -145.13 & -2891.75 & -164.01 & -269.51 & -1910.11 &  -3460.18 \\
    AntiVAE ($\log\mathcal{N}$) & -149.78 & \textbf{-141.76} & \textbf{-2882.11} & \textbf{-163.55} & \textbf{-266.82} & \textbf{-1895.15} & \textbf{-3454.54} \\
    \hline
    VAE (Exp.) & \textbf{141.95} & -140.91 & -2971.00 & -159.92 & -200.14 & -2176.83 & -3776.48 \\
    AntiVAE (Exp.) & 141.98 & \textbf{-140.58} & \textbf{-2970.12} & \textbf{-158.15} & -\textbf{197.47} & \textbf{-2156.93} & \textbf{-3770.33} \\
    \hline
    VAE (Cauchy) & -217.69 & -217.53 & -3570.53 & -187.34 & -419.78 & -2404.24 & -3930.40 \\
    AntiVAE (Cauchy) & \textbf{-215.89} & \textbf{-217.12} & \textbf{-3564.80} & \textbf{-186.02} & \textbf{-417.0} & \textbf{-2395.07} & \textbf{-3926.95} \\
    \hline
    VAE+10-NF & -90.07 & -86.93 & -2803.98 & -110.03 & -128.62 & -1780.61 & -3328.68 \\
    AntiVAE+10-NF & \textbf{-89.77} & \textbf{-86.57} & \textbf{-2801.90} & \textbf{-109.43} & \textbf{-127.23} & \textbf{-1777.26} & \textbf{-3303.00}\\
    \hline
    VAE+10-VPF & -90.59 & -86.99 & -2802.65 & -110.19 & -128.87& -1789.18 & -3312.30 \\
    AntiVAE+10-VPF & \textbf{-90.00} & \textbf{-86.59} & \textbf{-2797.05} & \textbf{-109.04} & \textbf{126.72} & \textbf{-1787.18} & \textbf{-3305.42}\\
\end{tabular}
\caption{Test log likelihoods between the VAE and AntiVAE under different objectives and posterior families (a higher number is better). Architecture and hyperparameters are consistent across models. AntiVAE (Cheng) refers to drawing antithetic sampling using an alternative algorithm to Marsaglia (see supplement). Results show the average over 5 independent runs with different random seeds. For measurements of variance, see supplement.}
\label{table:results}
\end{table*}

\begin{figure*}[t!]
    \begin{subfigure}[b]{0.16\textwidth}
        \includegraphics[width=\textwidth]{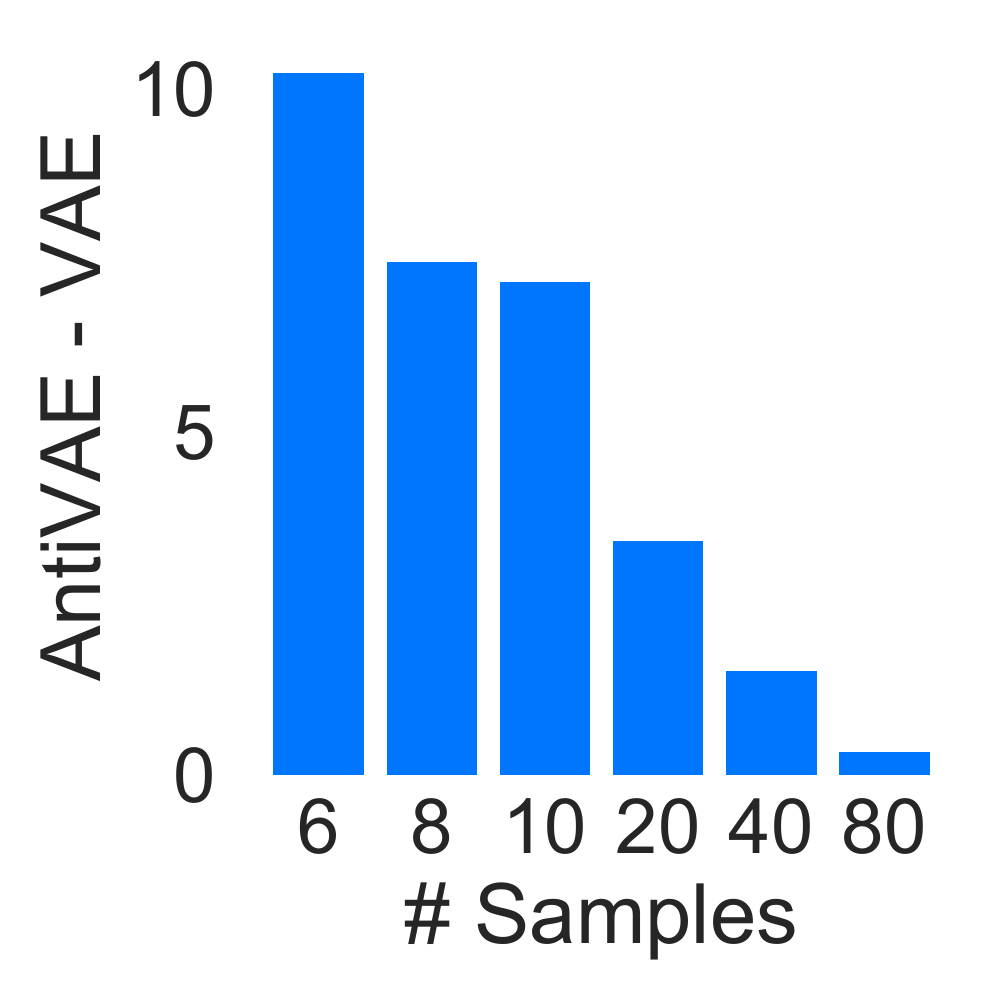}
        \caption{FashionMNIST}
    \end{subfigure}
    \begin{subfigure}[b]{0.16\textwidth}
        \includegraphics[width=\textwidth]{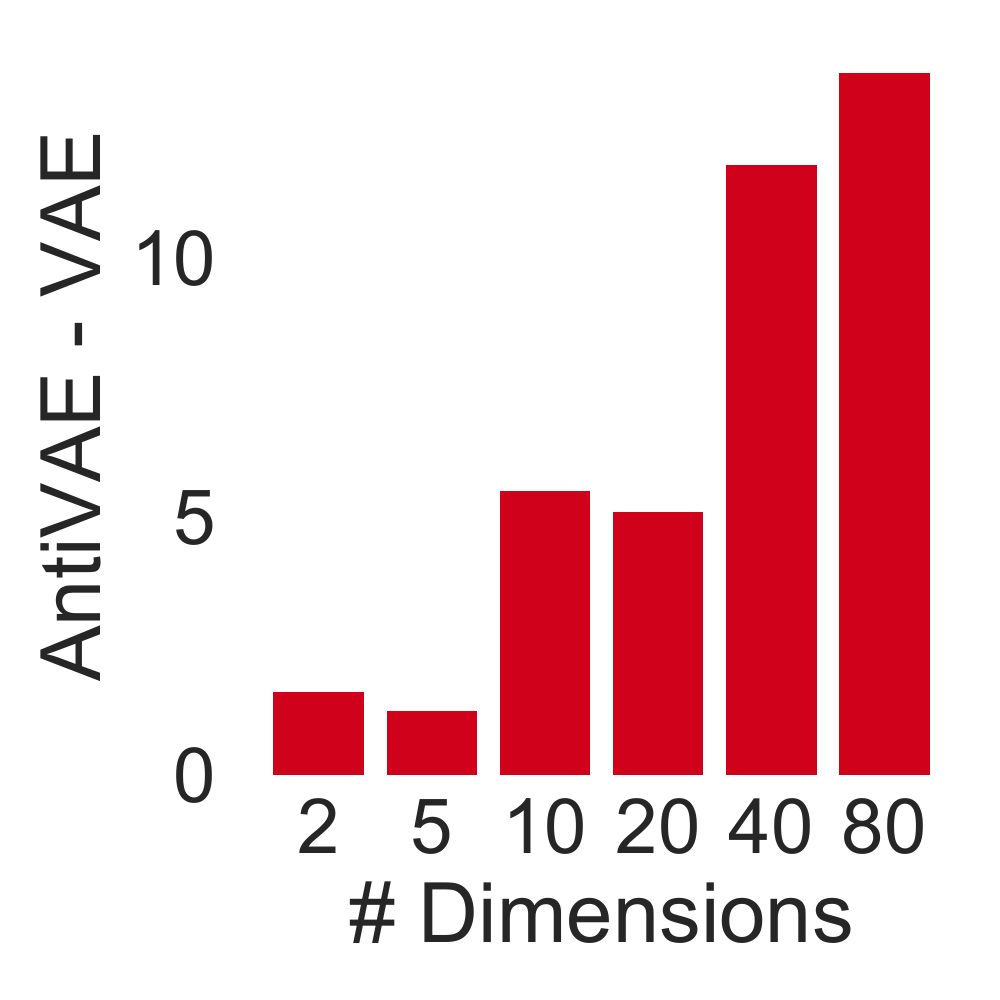}
        \caption{FashionMNIST}
    \end{subfigure}
    \begin{subfigure}[b]{0.24\textwidth}
        \includegraphics[width=\textwidth]{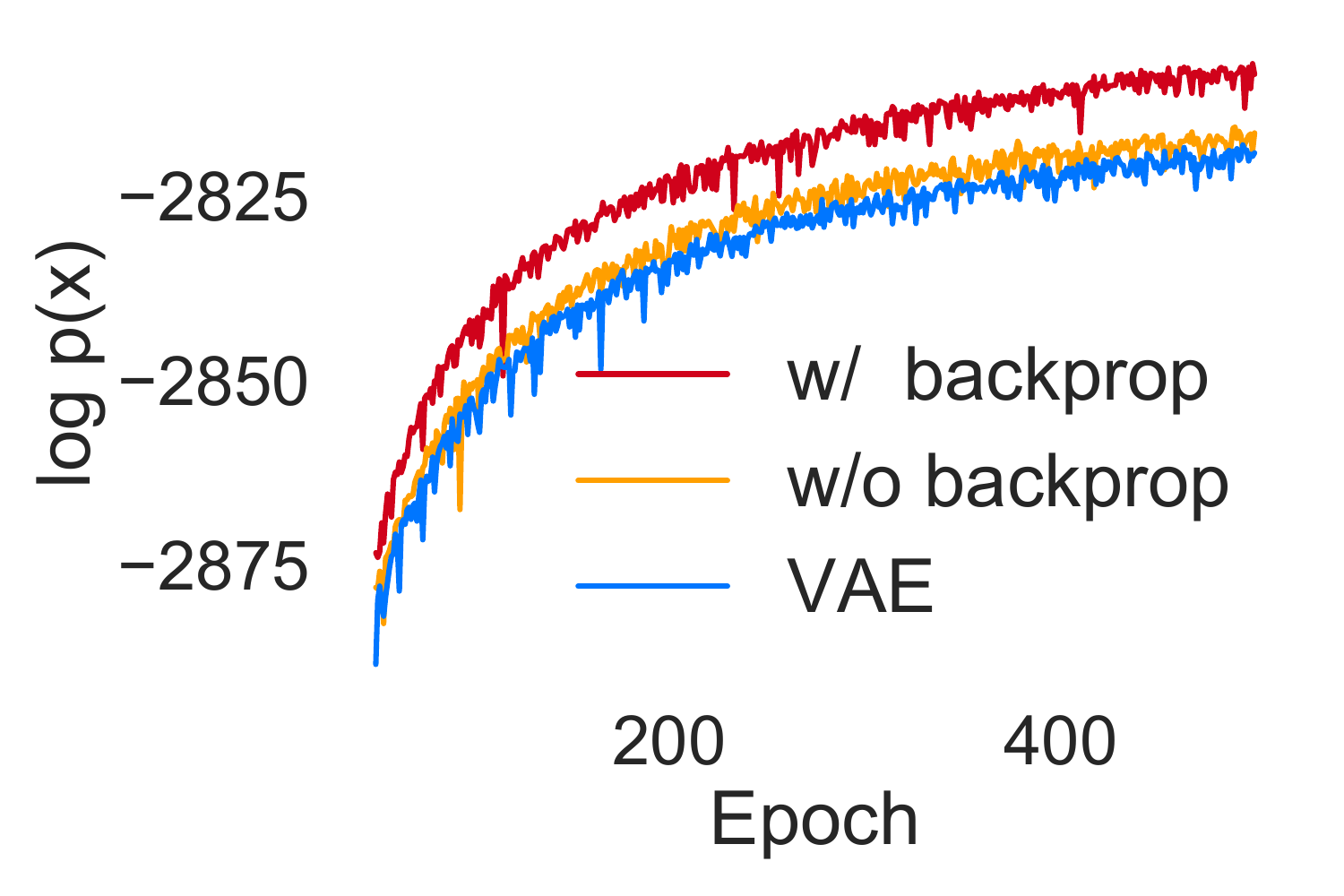}
        \caption{FashionMNIST}
        \label{fig:discussion:c}
    \end{subfigure}
    \begin{subfigure}[b]{0.24\textwidth}
        \includegraphics[width=\textwidth]{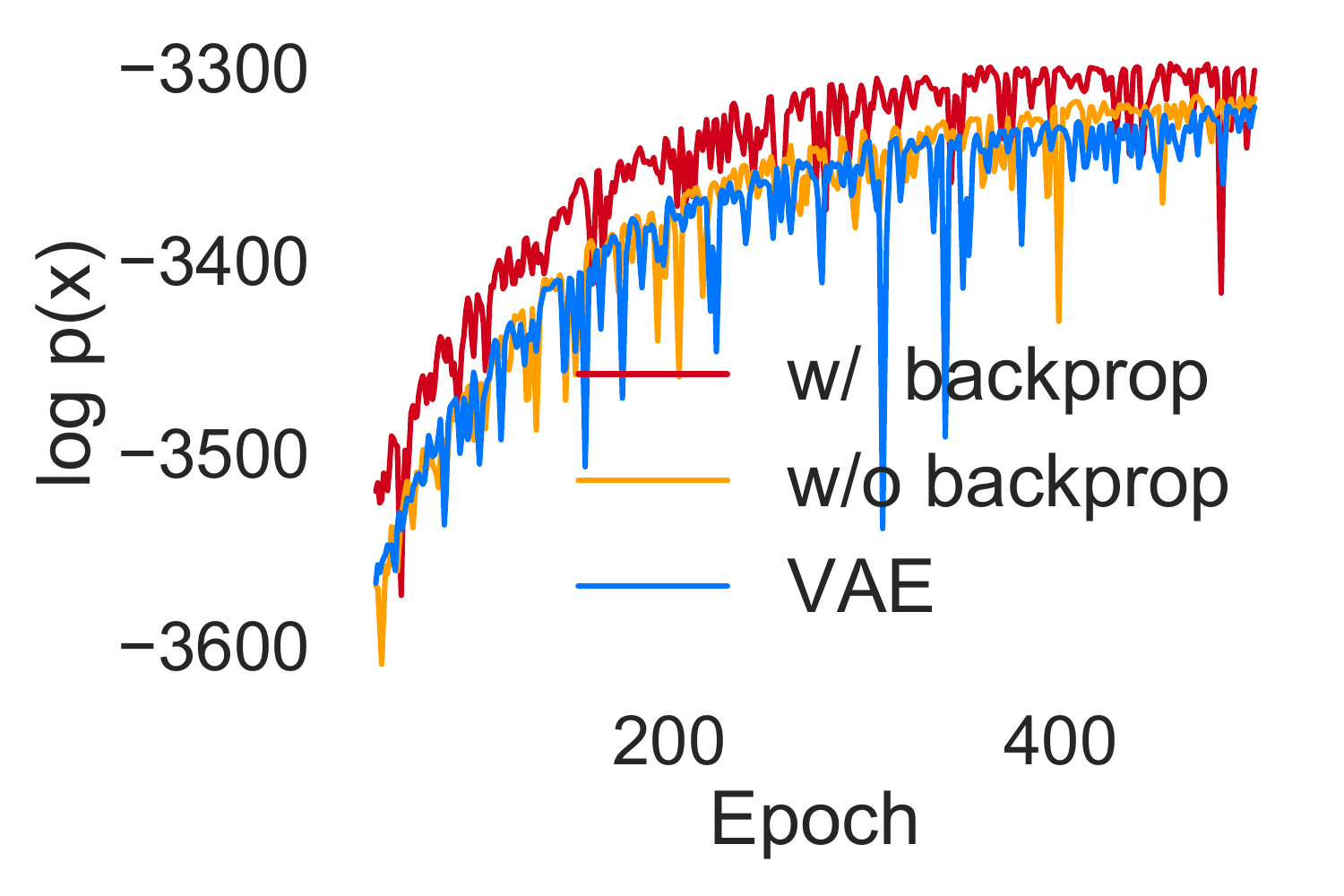}
        \caption{Histopathology}
        \label{fig:discussion:d}
    \end{subfigure}
    \begin{subfigure}[b]{0.16\textwidth}
        \includegraphics[width=\textwidth]{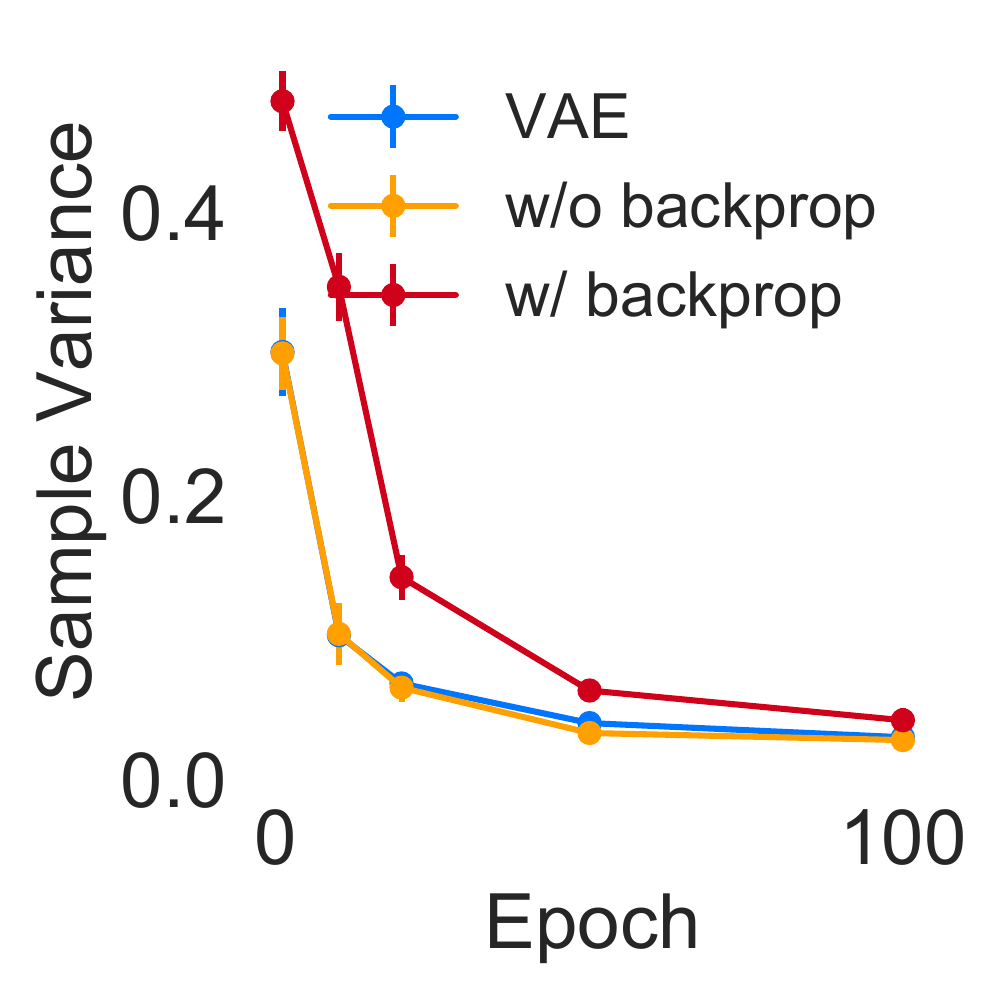}
        \caption{Histopathology}
        \label{fig:discussion:e}
    \end{subfigure}
    \caption{(a) With more samples, the difference in $\log p(x)$ between AntiVAE and VAE approaches 0. (b) The benefit of antithetics varies directly with dimensionality. (c) Backpropagating through \textsc{AntitheticSample} is responsible for most of the improvement over i.i.d. sampling. However, even without it, antithetics outperforms VAE. (d) Similar observation in Histopathology. (e) Differentiable antithetics encourages sample diversity.}
    \label{fig:discussion}
\end{figure*}

\section{Experiments}
\label{sec:experiments}


We compare performance of the VAE and AntiVAE on seven image datasets: static MNIST \citep{larochelle2011neural}, dynamic MNIST \citep{lecun1998gradient}, FashionMNIST \citep{xiao2017fashion}, OMNIGLOT \citep{lake2015human}, Caltech 101 Silhouettes \citep{marlin2010inductive}, Frey Faces\footnote{https://cs.nyu.edu/~roweis/data.html}, and Histopathology patches \citep{tomczak2016improving}. See supplement for details.

In both VAE and AntiVAE, $q_\phi(z|x)$ and $p_\theta(x|z)$ are two-layer MLPs with 300 hidden units, Xavier initialization \citep{glorot2010understanding}, and ReLU. By default, we set $d=40$ and $k=8$ (i.e. 4 antithetic samples) and optimize either the ELBO or IWAE. For grayscale images, $p_\theta(x|z)$ parameterize a discretized logistic distribution as in \citep{kingma2016improved, tomczak2017vae}. The log variance from $p_\theta(x|z)$ is clamped between -4.5 and 0.0 \citep{tomczak2017vae}. We use Adam \citep{kingma2014adam} with a fixed learning rate of $3\cdot 10^{-4}$ and a mini-batch of 128. We train for 500 epochs. Test marginal log likelihoods are estimated via importance sampling using 100 i.i.d.~samples. See supplement for additional experiments where we vary architectures, measure runtimes, report variance over many runs, and more.

\section{Results}
\label{sec:results}

Fig.~\ref{fig:learning_trajectory} and Table~\ref{table:results} show test log likelihoods (over 5 runs). We summarize findings below:

\paragraph{VAE vs AntiVAE} AntiVAE consistently achieves higher log likelihoods, usually by a margin of 2 to 5 log units. With FashionMNIST/Histopathology, the margin grows to as much as 30 log units. In the 3 cases that AntiVAE performs worse than VAE, the log-likelihoods are almost equal ($\leq$ 1 log unit). In Fig.~\ref{fig:elbo_dynamic_mnist}, we see a case where, even when the final performance is equivalent, AntiVAE learns faster. We find similar behavior using a tighter IWAE bound or other posterior families defined by one liners and flows. With the latter, we see improvements of up to 25 log units. A better sampling strategy is effective regardless of the choice of objective and distributional family.










\paragraph{As $k$ increases, the effect of antithetic sampling diminishes.}
Fig.~\ref{fig:discussion}a illustrates that as the number of samples $k \rightarrow \infty$, posterior samples will match the true moments of $q_\phi(z|x)$ regardless of the sampling strategy. But as $k\rightarrow 0$, the effectiveness grows quickly. We expect best performance at small (but not too small) $k$ where the normal approximation (Eqn.~\ref{eqn:anti_approx}) is decent and the value of antithetics is high.

\paragraph{As $d$ increases, the effect of antithetic sampling grows.} Fig.~\ref{fig:discussion}b illustrates that the importance of sampling strategy increases as the dimensionality grows due to an exponential explosion in the volume of the sample space. With higher dimensionality, we find antithetic sampling to be more effective.

\paragraph{Backpropagating through antithetic sampling greatly improves performance.} From Fig.~\ref{fig:discussion:c}, \ref{fig:discussion:d}, we see that most of the improvement from antithetics relies on differentiating through \textsc{AntitheticSample}. This is sensible as the model can adjust parameters if it is aware of the sampling strategy, leading to better optima. Even if we do not backpropagate through sampling (draw antithetic samples from $\mathcal{N}(0, 1)$ followed by standard reparameterization), we will still find modest improvement over i.i.d.~sampling.

We believe differentiability encourages initial samples to be more diverse. To test this, we measure the variance of the first $k/2$ samples (1) without antithetics, (2) with non-differentiable antithetics, and (3) with differentiable antithetics. Fig.~\ref{fig:discussion}e shows that samples in (3) have consistently higher variance than (1) or (2).

\paragraph{AntiVAE runtimes are comparable.} We measure an average 0.004 sec. increase in wallclock time per step when adding in antithetics.

\section{Conclusion}
\label{sec:conclusion}
We present a differentiable antithetic sampler for variance reduction. We show its benefits for a family of VAEs. We hope to apply it to reinforcement learning using pathwise derivatives~\citep{levy2018deterministic}.

\subsubsection*{Acknowledgements}
This research was supported by NSF (\#1651565, \#1522054, \#1733686), ONR, AFOSR (FA9550-19-1-0024), FLI, and SocioNeticus (part of the DARPA SocialSim program). MW is supported by NSF GRFP.


\bibliography{antithetic}

\onecolumn

\appendix
\section{Description of Datasets}

We provide details in preprocessing for datasets used in the experiments in the main text. In total, we tested AntiVAE on seven datasets: static MNIST, dynamic MNIST, FashionMNIST, OMNIGLOT, Caltech 101 Silhouettes, Frey Faces, and Histopathology patches. As in previous literature, static MNIST uses a fixed binarization of images whereas dynamic MNIST resamples images from the training dataset at each minibatch. In dynamic MNIST, the validation and test sets have fixed binarization. We do an identical dynamic resampling procedure for OMNIGLOT. Caltech101 is given as binary data, so we cannot resample at training time, which we find to cause overfitting on the test set. For grayscale images that cannot be binarized, we would like to parameterize the generative model as a Gaussian distribution. In practice, we found this choice to cause over-prioritization of the reconstruction term, essentially causing the VAE to behave like a regular autoencoder. Instead, we find that a logistic distribution over the 256 grayscale domain avoids this failure mode. We use the default variance constraints as in \citep{tomczak2017vae}. We now provide a brief description to introduce each dataset.

\paragraph{MNIST} is a dataset of hand-written digits from 0 to 9 split into 60,000 examples for training and 10,000 for testing. We use 10,000 randomly chosen images from the training set as a validation group.

\paragraph{FashionMNIST} Similar to MNIST, this more difficult dataset contains 28x28 \textit{grayscale} images of 10 different articles of clothing e.g. skirts, shoes, shirts, etc. The sizing and splits are identical to MNIST.

\paragraph{OMNIGLOT} is a dataset with 1,623 hand-wrriten characters from 50 different alphabets. Unlike the MNIST family of datasets, each character is only represented by 20 images, making this dataset more difficult. The training set is 24,345 examples with 8,070 test images. We again take 10\% of the training data as validation.

\paragraph{Caltech 101 Silhouettes} contains silhouettes of 101 different object classes in black and white: each image has a filled polygon on a white background. There are 4,100 training images, 2,264 validation datapoints and 2,307 test examples. Like OMNIGLOT, this task is difficult due to the limited data set.

\paragraph{FreyFaces} is a collection of portrait photos of one individual with varying emotional expressions for a total of 2,000 images. We use 1,565 for training, 200 validation, and 200 test examples.

\paragraph{Histopathology Patches} is a dataset from ten different biopsies of patients with cancer (e.g. lymphona, leukemia) or anemia. The dataset is originally in color with 336 x 448 pixel images. The data was processed to be 28 x 28 grayscale. The images are split in 6,800 training, 2,000 validation, and 2,000 test images. We refer to \citep{tomczak2016improving} for exact details.

All splitting was either given by the dataset owners or decided by a random seed of 1.

\subsection{Evaluation Details} To compute the test log-likelihood (in any of the experiments), we use $k=100$ samples to estimate the following:

\begin{equation}
    \log \mathbb{E}_{z \sim q_\phi(z|x)}[\frac{p_\theta(x, z)}{q_\phi(z|x)}] \approx \log \frac{1}{k}\sum_{i=1}^{k}[\frac{p_\theta(x, z)}{q_\phi(z|x)}]
    \label{eqn:marginal}
\end{equation}

where $q_\phi(z|x)$ is an amortized inference network, $p_\theta(x|z)$ is a generative model, and $p(z)$ is a simple prior. Notably, we use (unbiased) i.i.d. samples to estimate Eqn.~\ref{eqn:marginal}. The final number reported is the average test log likelihood across the test split of the dataset.

\subsection{Approximate Antithetic Sampling Algorithm}

We explicitly write out the approximate algorithm for antithetic sampling. Note the similiarity to Alg. 2 in the main text; the only distinction is that we use a derived approximation to the inverse CDF tranform for a Chi-squared random variable.. Here, we refer to this as \textsc{ApproxAntitheticSample}. In the main text, this algorithm is often referred to as \textsc{AntitheticSample}.

\begin{algorithm}[h!]
\SetAlgoLined
\caption{\textsc{ApproxAntitheticSample}}
\KwData{i.i.d. samples $(x_1, ..., x_k) \sim \mathcal{N}(\mu, \sigma^2)$; i.i.d. samples $\boldsymbol{\epsilon} = (\epsilon_1, ..., \epsilon_{k-1}) \sim \mathcal{N}(0, 1)$; Population mean $\mu$ and variance $\sigma^2$; Number of samples $k \in \mathbb{N}$.} 
\KwResult{A set of $k$ samples $x_{k+1}, x_{k+2}, ..., x_{2k}$ marginally distributed as $\mathcal{N}(\mu, \sigma^2)$ with sample mean $\eta'$ and sample standard deviation $\delta'$.}

$v = k -1$\;
$\eta = \frac{1}{k}\sum_{i=1}^k x_i$\;
$\delta^2 = \frac{1}{k}\sum_{i=1}^k (x_i - \eta)^2$\;
$\eta' = 2\mu - \eta$\;
$\lambda = v\delta^2/\sigma^2$\;
$\lambda' = v(2(1 - \frac{3}{16v} - \frac{7}{512v^2} + \frac{231}{8192v^3}) - (\frac{\lambda}{v})^{1/4})^4$\;
$(\delta')^2 = \lambda'\sigma^2/v$\;
$(x_{k+1}, ..., x_{2k}) = \textsc{MarsagliaSample}(\boldsymbol{\epsilon}, \eta', (\delta')^2, k)$\;
Return $(x_{k+1}, ..., x_{2k})$\;
\label{alg:antithetic}
\end{algorithm}

\section{Proofs of Propositions}
In this section, we provide more rigor in proving (1) properties of antithetic samplers using Marsaglia's method and (2) properties of Marsaglia's method itself. In particular, we provide the proof of Proposition 1 (from the main text).

\subsection{Properties of Antithetic Sampling}

\begin{theorem} Let $p(x)$ be a distribution over $\mathbb{R}$ and let $p(\zeta) = \prod_i^k p(x_i)$ be the distribution of $k$ i.i.d.~samples. Let $T: \mathbb{R}^k \rightarrow \mathbb{R}^l$ be a ($l$-dimensional) statistic of $\zeta$. Let $s(t)$ be the induced distribution of this statistic:  $s(t) = \int_\zeta \delta_{T(\zeta){=}t} p(\zeta)d\zeta$.
Let $F: \mathbb{R}^l \rightarrow \mathbb{R}^l$ be a deterministic function such that $s(F(t)) = s(t)$.
We now construct a sample $\bar{\zeta}$ by: sampling $\zeta \sim p(\zeta)$, computing $\bar{t}=F(T(\zeta))$, sampling $\bar{\zeta} \sim p(\zeta|\bar{t})$ from the conditional given $T(\zeta)=\bar{t}$.
This $\bar{\zeta}$ is distributed according to $p(\zeta)$ and, in particular, it's elements are i.i.d.~according to $p(x)$.
\begin{proof}
We begin by noting that:
\begin{equation}
    p(\zeta) = \int_t p(\zeta|t)s(t)
\label{eqn:1}
\end{equation}
By assumption,
\begin{align}
    s(\bar{t}) &= s(F(t)) \\
        &= s(t).
\end{align}
Thus,
\begin{align}
    p(\bar{\zeta}) &= \int_{\bar{t}} p(\zeta|\bar{t})s(\bar{t})\\
             &= \int_{t} p(\zeta|t)s(t) \\
             &= p(\zeta)
\label{eqn:2}
\end{align}
Thus $\bar{\zeta} \sim p(\zeta)$.
Since $p(\zeta)$ is the distribution over i.i.d.~samples from $p(x)$, the resulting elements of $\bar{\zeta}$ are also i.i.d.~from $p(x)$.
\end{proof}
\label{theorem:1}
\end{theorem}

We provide one example of a function $F$ with the desired property.

\begin{lemma}Let $F(t) = \textsc{CDF}(1 - \textsc{CDF}^{-1}(t))$ where \textsc{CDF} is the cumulative distribution function for $s(t)$. Then $s(F(t))= s(t)$.
\end{lemma}
\begin{proof}
Let $X \sim \textsc{U}(0,1)$. By definition, \textsc{CDF}$(X)$ will be distributed as $s(t)$. Trivially, $\textsc{CDF}^{-1}(t) \sim \textsc{U}(0,1)$ when $t \sim s(t)$, and so too is $1-\textsc{CDF}^{-1}(t)$.
\end{proof}

\begin{corollary}
Let $\theta = \mathbb{E}_p[h(x)]$ be a function  expectation of interest with respect to a distribution, $p(x), x \in \mathbb{R}$. Let $\hat{\theta}_1$ be an unbiased Monte Carlo estimate using i.i.d. samples $\zeta \sim p(\zeta)$. Let $\hat{\theta}_2$ be an ``antithetic" estimate using samples $\bar{\zeta}$ generated as in Theorem~\ref{theorem:1}. Then the following hold,

\begin{itemize}
    \item $\hat{\theta}_2$ is unbiased estimate of $\theta$
    \item $\hat{\theta}_3 = \frac{\hat{\theta}_1 + \hat{\theta}_2}{2}$  is unbiased estimate of $\theta$
    \item Let $F = \textsc{CDF}(1 - \textsc{CDF}^{-1}(T))$. Then the first and second moments of $\zeta$ are anti-correlated to those of $\bar{\zeta}$
\end{itemize}
\label{corollary:1}
\end{corollary}

\begin{proof}
By Theorem~\ref{theorem:1}, ``antithetic" samples $\bar{\zeta} \sim p(\zeta)$ i.i.d., hence $\hat{\theta}_2$ is unbiased ($\hat{\theta}_2$ is equivalent to  $\hat{\theta}_1$). $\hat{\theta}_3$ is also  unbiased as a linear combination of two unbiased estimators is itself unbiased. Anti-correlation of moments falls trivially from our choice of $F$.
\end{proof}

\paragraph{Connection to Differentiable Antithetic Sampling}
In the paper, we proposed the following proposition,

\begin{prop}
For any $k > 2$, $\mu \in \mathbb{R}$ and $\sigma^2 \in \mathbb{R}^{+}$, if $\eta \sim \mathcal{N}(\mu, \frac{\sigma^2}{k})$ and $\frac{(k-1)\delta^2}{\sigma^2} \sim \chi^2_{k-1}$, and $\bar{\eta} = f(\eta), \bar{\delta}^2 = g(\delta^2; \sigma^2)$ for some functions $f: \mathbb{R} \rightarrow \mathbb{R}$ and $g: \mathbb{R} \rightarrow \mathbb{R}$, and $\epsilon = (\epsilon_1, ..., \epsilon_k) \sim \mathcal{N}(0, 1)$, then the ``antithetic" samples $\zeta = (x_1, ..., x_k) = \textsc{MarsagliaSample}(\epsilon, \bar{\eta}, \bar{\delta}^2, k)$ are independent normal variates sampled from $\mathcal{N}(\mu, \sigma^2)$ such that $\frac{1}{k}\sum_i^k x_i = \bar{\eta}$ and $\frac{1}{k}\sum_i^k (x_i - \bar{\eta})^2 = \bar{\delta}^2$.
\label{prop:3}
\end{prop}

Define a statistic $T  = [\bar{\eta}, \bar{\delta}^2]$, and function $F = [f, g]$. Marsaglia's algorithm (or Pullin's, Cheng's) can be seen as a method for sampling from $p(\zeta|t)$ for a fixed statistic $t$. In Proposition~\ref{prop:3}, we first sample $t \sim s(T(\zeta))$ where $\zeta \sim \mathcal{N}(\mu, \sigma^2)$. Then, we choose ``antithetic" statistics using
\begin{align}
    f & = \textsc{GaussianCDF}(1 - \textsc{GaussianCDF}^{-1}(\eta)) \\
    g & = \frac{\sigma^2}{(k-1)}\textsc{ChiSquaredCDF}(1 - \textsc{ChiSquaredCDF}^{-1}(\frac{(k-1)\delta^2}{\sigma^2}))
\end{align}
such that $s(F(t)) = s(t)$ by symmetry in $\textup{U}(0, 1)$. By Theorem~\ref{theorem:1}, antithetic samples $\bar{\zeta}$ are distributed as $\zeta$ is. In practice, we use both $\zeta$ and $\bar{\zeta}$ for stochastic estimation, as anti-correlated moments provide empirical benefits.

\subsection{Properties of Marsaglia's Method}

\begin{theorem}
Let $\epsilon = (\epsilon_1, ..., \epsilon_{k-1}) \sim \mathcal{N}(0, 1)$ auxiliary variables.  Let $\eta, \delta$ be known variables. Then $\zeta = (x_1, ..., x_k) = \textsc{MarsagliaSample}(\epsilon, \eta, \delta^2, k)$ are uniform samples from the sphere
\[S = \{(x_1, ..., x_k) | \sum_i^k x_i = k\eta, \sum_i^k(x_i - \eta)^2 = k \delta^2\}\]
\begin{proof}
$S$ is the intersection of a hyperplane and the surface of a $k$-sphere: the surface of a $(k-1)$-sphere. Marsaglia uses the following to sample from $S$:

Let $z = (z_1, ..., z_{k-1})$ be a sample drawn uniformly from the unit $(k-1)$-sphere centered at the origin. (In practice, set $z_i = \epsilon_i / \sqrt{\sum_j^k \epsilon_j^2}$.) Let
\begin{equation}
    \zeta = rzB + \eta v
\end{equation}
where $v = (1, 1, ..., 1)$ and choose $B$ to be a $(k-1)$ by $k$ matrix whose rows form an orthonormal basis with the null space of $v$. By definition, $BB^t = I$ and $Bv^t = 0$ where $I$ is the identity matrix. We note the following consequence:
\begin{align}
    \zeta v^t &= (rzB + \eta v)v^t\\
                 &= rzBv^t + \eta vv^t\\
                 &= 0 + \eta vv^t\\
                 &= k\eta
\label{eqn:p1}
\end{align}
\begin{align}
    (\zeta - \eta v)(\zeta - \eta v)^t &=
    (rzB + \eta v - \eta v)(rzB + \eta v - \eta v)^t\\
    &= (rzB)(rzB)^t \\
    &= r^2zBB^tz^t \\
    &= r^2zz^t \\
    &= r^2
\label{eqn:p2}
\end{align}

Eqn.~\ref{eqn:p1},~\ref{eqn:p2} exactly match the constraints defined in $S$. So $\zeta \in S$. Further $\zeta$ is uniformly distributed in S as $z$ is uniform over the $(k-1)$-sphere.
\end{proof}
\label{theorem:2}
\end{theorem}

\begin{theorem} Let  $\zeta
= (x_1, ..., x_k)
\sim p(\zeta)$  be a random vector of i.i.d. Gaussians $\mathcal{N}(\mu, \sigma^2)$. Let $\eta=\frac{1}{k}\sum_i^k x_i$ and $\delta^2 = \frac{1}{k}\sum_i^k (x_i - \eta)^2$. Then
$\eta \sim \mathcal{N}(\mu, \frac{\sigma^2}{k})$ and $\frac{(k-1)\delta^2}{\sigma^2} \sim \chi^2_{k-1}$ and $\eta,\delta^2$ are independent random variables.
\begin{proof}
This is a known property of Gaussian distributions. Reference \textit{Statistics: An introductory analysis} or any introductory statistics textbook.
\end{proof}
\label{marginal:1}
\end{theorem}

\begin{theorem} Let  $\zeta=(x_1, ..., x_k)$  be a random vector of i.i.d. Gaussians $\mathcal{N}(\mu, \sigma^2)$. Let $\eta=\frac{1}{k}\sum_i^k x_i = $ and $\delta^2 = \frac{1}{k}\sum_i^k (x_i - \eta)^2$ and $T  = [\eta, \delta^2]$. Let $p(\zeta, T(\zeta)) = p(\zeta, \eta,\delta^2)$ denote their joint distribution.

Then, the conditional density is of the form
    \begin{equation}
        p(\zeta | \eta=\eta,\delta^2 = \delta^2 ) = \begin{cases}
 & a \text{ if } \zeta \in S \\
 & 0 \text{ if } \zeta \notin S.
\end{cases}
\end{equation}
where $S = \{(x_1, ..., x_k) | \sum_i x_i = k\eta, \sum_i(x_i - \eta)^2 = k \delta^2\}$,  $0 < a < 1$ is a constant.
\begin{proof}
\noindent\newline\textbf{Intuition:} Level sets of a multivariate isotropic Gaussian density function are spheres. The event we are conditioning on is a sphere.

\noindent\newline\textbf{Formal Proof:}    Let $f(x_1, ...,x_k) = (2\pi\sigma^2)^{-k/2} e^{(-\sum_i(x_i - \mu)^2 / (2\sigma^2))}$ denote a Gaussian density. Note the following derivation:

    \begin{align}
        \sum_{i=1}^k(x_i - \mu)^2 &= \sum_i(x_i - \eta)^2 + 2(\eta - \mu)\sum_i(x_i - \eta) + k(\eta - \mu)^2\\
                            &= \sum_i(x_i - \eta)^2 + k(\eta - \mu)^2\\
                            &= r^2 + k(\eta - \mu)^2\label{eqn:constant}
    \end{align}
 This implies $f(x_1, ..., x_k)$ is equal for any $(x_1, ..., x_k) \in S$. Thus, the conditional distribution $p(\zeta | \zeta \in S)$ is the uniform distribution over $S$ for any $\mu, \sigma$.
\end{proof}
\label{conditional:1}
\end{theorem}

Finally, proof of Proposition 1 from the paper (denoted as Proposition~\ref{prop:marsaglia} here):
\begin{prop}
    For any $k>2$, $\mu \in \mathbb{R}$ and $\sigma^2>0$, if $\eta \sim \mathcal{N}(\mu, \frac{\sigma^2}{k})$ and $\frac{(k-1)\delta^2}{\sigma^2} \sim \chi^2_{k-1}$ and $\boldsymbol{\epsilon} = \epsilon_1, ..., \epsilon_{k-1} \sim \mathcal{N}(0, 1)$ i.i.d., then the generated samples $x_1, ..., x_k = \textsc{MarsagliaSample}(\boldsymbol{\epsilon}, \eta, \delta^2,k)$ are independent normal variates sampled from $\mathcal{N}(\mu, \sigma^2)$ such that $\frac{1}{k}\sum_i x_i = \eta$ and $\frac{1}{k}\sum_i(x_i - \eta)^2 = \delta^2$.
    \label{prop:marsaglia}
\end{prop}
\begin{proof}
Let  $\zeta=(x_1, ..., x_k)$  be a random vector of i.i.d. Gaussians $\mathcal{N}(\mu, \sigma^2)$. Compute $\eta=\frac{1}{k}\sum_i^k x_i$ and $\delta^2 = \frac{1}{k}\sum_i^k (x_i - \eta)^2$ and $T  = [\eta, \delta^2]$. Let $p(\zeta, T(\zeta)) = p(\zeta, \eta,\delta^2)$ denote their joint distribution.  Factoring
\[
p(\zeta, \eta, \delta^2) = p(\eta,\delta^2) p(\zeta \mid \eta, \delta^2)
\],
it is clear that we can sample from the joint by first sampling $\eta,\delta^2 \sim p(\eta, \delta^2)$ and then $\zeta' \sim p(\zeta \mid \eta=\eta,\delta^2=\delta^2)$. From Theorem \ref{marginal:1}, we know $p(\eta, \delta^2)$ analytically and from Theorem \ref{conditional:1} we know $p(\zeta \mid \eta,\delta^2)$ is a uniform distribution over the sphere. By assumption, $\eta,\delta^2$ are sampled independently from the correct marginal distributions from Theorem \ref{marginal:1}. Then, from Theorem \ref{theorem:2}, we know $\textsc{MarsagliaSample}(\epsilon, \eta, \delta^2, k)$ samples from the correct conditional density (i.e. from S). Thus, samples $\zeta'$ from \textsc{MarsagliaSample} will have the same distribution as $\zeta$, namely i.i.d. Gaussian.
\end{proof}

\section{Additional Experiments}

\subsection{Convolutional Architectures}
In the main text, we present results where $q_\phi(z|x)$ and $p_\theta(x|z)$ are parameterized by feedforward neural networks (multilayer perceptrons). While that architecture choice was made for simplicitly, we recognize that modern encoder/decoders have evolved beyond linear layers. Thus, we ran a subset of the experiments using DCGAN architectures \citep{radford2015unsupervised}. Specifically, we design $q_\phi(z|x)$ using 3 convolutional layers and $p_\theta(x|z)$ with 3 deconvolutional layers and 1 convolutional layer.

\begin{table}[h!]
\small
\centering
\begin{tabular}{r|ccccccc}
    Model & stat. MNIST & dyn. MNIST & FashionMNIST & Omniglot & Caltech  & Hist. \\
    \toprule
    VAE & -90.58 & -90.02 & -2767.97 & -108.97 & -116.15 &  -3218.16\\
    AntiVAE & -90.25 & -89.53 & -2762.02 & -108.40 & -115.14 & -3213.83 \\
    VAE+IWAE & -89.19 & -88.61 & -2758.72 & -107.52 & -116.25 & -3213.05 \\
    AntiVAE+IWAE & -89.01 & -88.13 & -2751.11 & -107.44 & -115.04 & -3209.98 \\
\end{tabular}
\caption{Test log likelihoods between the VAE and AntiVAE using (de)convolutional architectures for encoders and decoders. All images were reshaped to 32 by 32 to match standard DCGAN input sizes. }
\label{table:results_conv}
\end{table}

Table~\ref{table:results_conv} shows log-likelihoods on a test set for a variety of image datasets. Like experiments presented in the main text, we find improvements in density estimation when using antithetics. This agrees with our intuition that more representative samples benefit learning regardless of architecture choice.

\subsection{Variance over Independent Runs}
In the main text, we report the average test log likelihoods over 5 runs, each with a different random seed. Here, we report in Table.~\ref{table:error} the variance as well (which we could not fit in the main table).

\begin{table}[h!]
\tiny
\centering
\begin{tabular}{r|cccccc}
    Dataset & VAE & AntiVAE & VAE+IWAE & AntiVAE+IWAE & VAE+10-NF & AntiVAE+10-NF \\
    \toprule
    StaticMNIST & $-90.44 \pm 0.031$ & $-89.74 \pm 0.066$ & $-89.78 \pm 0.080$ & $-89.71 \pm 0.059$ & $-90.07 \pm 0.033$ & $-89.77 \pm 0.042$ \\
    DynamicMNIST & $-86.96 \pm 1.398$ & $-86.94 \pm 1.412$ & $-86.71 \pm 1.778$ & $-86.62 \pm 1.426$ & $-86.93 \pm 1.132$ & $-86.57 \pm 1.173$ \\
    FashionMNIST & $-2819.13 \pm 1.769$ & $-2807.06 \pm 1.591$ & $-2797.02 \pm 1.714$ & $-2793.01 \pm 1.174$ & $-2803.98 \pm 1.487$ & $-2801.90 \pm 1.459$ \\
    Omniglot & $-110.65 \pm 0.141$ & $-110.13 \pm 0.063$ & $-109.32 \pm 0.134$ & $-109.48 \pm 0.104$ & $-110.03 \pm 0.178$ & $-109.43 \pm 0.057$ \\
    Caltech101 & $-127.26 \pm 0.254$ & $-124.87 \pm 0.213$ & $-123.99 \pm 0.262$ & $-123.35 \pm 0.195$ & $128.62 \pm 0.278$ & $-126.72 \pm 0.247$ \\
    FreyFaces & $-1778.78 \pm 4.649$ & $-1758.66 \pm 7.581$ & $-1772.06 \pm 7.275$ & $-1771.47 \pm 5.783$ & $-1780.61 \pm 4.595$ & $-1777.26 \pm 6.467$ \\
    Histopathology & $-3320.37 \pm 6.136$ & $-3294.23 \pm 1.543$ & $-3311.23 \pm 2.859$ & $-3305.91 \pm 1.972$ & $-3328.68 \pm 5.426$ & $-3303.00 \pm 1.517$ \\
    \end{tabular}
\caption{Identical to Table 1 in the main text but we include an errorbar over 5 runs. We find the differences induced by antithetics to be significant.}
\label{table:error}
\end{table}

\subsection{Runtime Experiments}
To measure runtime, we compute the average wall-time of the forward and backward pass over a single epoch with fixed hyperparameters for VAE and AntiVAE. Namely, we use a minibatch size of 128 and vary the number of samples $k=8, 16$. The measurements are in seconds using a Titan X GPU with CUDA 9.0. The implementation of the forward pass in PyTorch is vectorized across samples for both VAE and AntiVAE. Thus the comparison of runtime should be fair. We report the results in the Table.~\ref{table:runtime}.

\begin{table}[h!]
\tiny
\centering
\begin{tabular}{r|r|ccccccc}
    $k$ & Model & StaticMNIST & DynamicMNIST & FashionMNIST & OMNIGLOT & Caltech101  & FreyFaces & Hist. Patches \\
    \toprule
    8 & VAE & $0.0132 \pm 0.011$ & $0.0122 \pm 0.010$ & $0.0142 \pm 0.009$ & $0.0144 \pm 0.015$ & $0.0188 \pm 0.034$ & $0.0283 \pm 0.052$ & $0.0173 \pm 0.028$ \\
    8 & AntiVAE & $0.0179 \pm 0.011$ & $0.0156 \pm 0.009$ & $0.0173 \pm 0.010$ & $0.0164 \pm 0.017$ & $0.0220 \pm 0.036$ & $0.0334 \pm 0.054$ & $0.0196 \pm 0.029$ \\
    8 & AntiVAE (Cheng) & $0.0242 \pm 0.014$ & $0.0210 \pm 0.010$ & $0.0231 \pm 0.009$ & $0.0221 \pm 0.015$ & $0.0353 \pm 0.040$ & $0.040 \pm 0.062$ & $0.0303 \pm 0.026$\\
    \hline
    16 & VAE & $0.0228 \pm 0.009$ & $0.0182 \pm 0.011$ & $0.0207 \pm 0.010$ & $0.0181 \pm 0.015$ & $0.0275 \pm 0.035$ & $0.0351 \pm 0.049$ & $0.0245 \pm 0.027$\\
    16 & AntiVAE & $0.0252 \pm 0.009$ & $0.0240 \pm 0.011$ & $0.0288 \pm 0.010$ & $0.0256 \pm 0.015$ & $0.0308 \pm 0.035$ & $0.0384 \pm 0.049$ & $0.0315 \pm 0.027$ \\
    16 & AntiVAE (Cheng) & $0.0388 \pm 0.011$ & $0.0396 \pm 0.010$ & $0.0452 \pm 0.011$ & $0.0399 \pm 0.015$ & $0.0461 \pm 0.038$ & $0.0550 \pm 0.054$ & $0.0505 \pm 0.033$ \\
\end{tabular}
\caption{A comparison of runtime estimates between VAE and AntiVAE over different datasets. The number reported is the number of seconds for 1 forward and backward pass of a minibatch of size 128.}
\label{table:runtime}
\end{table}

To compute the additional cost of antithetic sampling, we divided the average runtimes of AntiVAE by the average runtimes of VAE and took the mean, resulting in 22.8\% increase in running time (about 0.004 seconds). We note that AntiVAE (Cheng) is much more expensive as it is difficult to vectorize Helmert's transformation.

\subsection{Importance of Differentiability}
We report the numbers plotted in Fig.4e, which showed that differentiability in antithetic sampling is the driving force behind sample diversity. The numbers reported are averaged over 5 runs on Histopathology.

\begin{table}[h!]
\small
\centering
\begin{tabular}{l|ccc}
    Epoch & VAE & AntiVAE (no backprop) & AntiVAE (with backprop) \\
    \toprule
    1 & $0.302 \pm  0.031$ & $0.301 \pm 0.026$ & $0.479 \pm 0.021$\\
    10 & $0.102 \pm 0.008$ & $0.103 \pm 0.022$ & $0.348 \pm 0.024$ \\
    20 & $0.068 \pm 0.006$ & $0.065 \pm 0.010$ &  $0.143 \pm 0.016$ \\
    50 & $0.040 \pm 0.005$ & $0.033 \pm 0.006$ & $0.063 \pm 0.004$ \\
    100 & $0.030 \pm 0.002$ & $0.028 \pm 0.008$ & $0.042 \pm 0.009$ \\
\end{tabular}
\caption{Variance of the first $k/2$ samples (non-antithetics) as measured over five independent runs on Histopathology. Without backprop, the variance is roughly equivalent to regular VAE.}
\label{table:diff_study}
\end{table}

As an aside, we provide the following remark: it is important to check that by adding differentiability, we do not introduce any unintended effects. For example, one might ask if differentiability leads to collapse of the VAE to a deterministic autoencoder (AE), thereby learning to ``sample" only the mean. To confirm that this is not the case, we measure the average variance (across dimensions and examples in the test set) of the variational posterior $q(z|x)$ when trained as a VAE versus as a AntiVAE.

\begin{table}[h!]
\small
\centering
\begin{tabular}{l|cc}
    Dataset & VAE & AntiVAE \\
    \toprule
    StaticMNIST & 0.253 & 0.290 \\
    DynamicMNIST & 0.269 & 0.290 \\
    FashionMNIST & 0.049 & 0.049 \\
    OMNIGLOT & 0.208 & 0.285\\
    Caltech101 & 0.179 & 0.182\\
    FreyFaces & 0.048 & 0.061\\
    Histopathology & 0.029 & 0.028\\
\end{tabular}
\caption{Learned variance of the approximate Gaussian posterior with and without antithetics. We measure variance on a variety of datasets.}
\label{table:collapse_study}
\end{table}

If differentiating through antithetic sampling led to ignoring noise, we would expect q(z|x) to be deterministic i.e. near 0 variance. This does not appear to be the case, as shown in Table.~\ref{table:collapse_study}.

\section{Deriving One-Liner Transformations}

We provide a step-by-step derviation for $g(\cdot)$ in one-liner transformations, namely from Gaussian to Cauchy and Exponential. We skip Log Normal as its formulation from a Gaussian variate is trivial. Below, let $X$ represent a normal variate and let $Y$ be a random variable in the desired distribution family.

\paragraph{Exponential} Let $F(X) = 1 - \exp^{-\lambda X}$. Parameters: $\lambda$.

We start with $F(F^{-1}(Y)) = Y$.

\begin{align*}
    1 - \exp^{-(\lambda F^{-1}(Y))} &= Y \\
    \exp^{-(\lambda F^{-1}(Y))} &= 1 - Y \\
    -(\lambda F^{-1}(Y)) &= \log (1 - Y) \\
   \lambda F^{-1}(Y) &= -\log (1 - Y) \\
   F^{-1}(Y)& = -\frac{1}{\lambda} \log (1 - Y) \\
\end{align*}

Since $1 - Y \in \text{U}(0, 1)$ and $Y \in \text{U}(0, 1)$, we can replace $1 - Y$ with $Y$.

\begin{equation*}
    F^{-1}(Y) = -\frac{1}{\lambda}\log Y \\
\end{equation*}

\paragraph{Cauchy}
Let $F(X) = \frac{1}{2} + \frac{1}{\pi}\arctan(\frac{X - x_0}{\gamma})$. Parameters: $x_0$, $\gamma$.

\begin{align*}
    \frac{1}{2} +   \frac{1}{\pi}\arctan(\frac{F^{-1}(Y) - x_0}{\gamma}) &= Y \\
    \arctan(\frac{F^{-1}(Y) - x_0}{\gamma}) &= \pi(Y - \frac{1}{2}) \\
    F^{-1}(Y) &= \gamma(\tan(\pi(Y - \frac{1}{2})) + x_0) \\
    F^{-1}(Y) &= \gamma(\tan(\pi Y) + x_0) \\
\end{align*}

In practice, we only optimize over $\gamma$, fixing $x_0$ to be 0.

\section{Deriving Antithetic Hawkins-Wixley}

We provide the following derivation for computing an antithetic $\chi^2$ variate using a normal approximation to the $\chi^2$ distribution. We assume the reader is familiar with the inverse CDF transform (as reviewed in the main text).

\citep{hawkins1986note} presented the following fourth root approximation of a $\chi^2_n$ variate, denoted $X^{(1)}$ with $n$ degrees of freedom as distributed according to the following Gaussian:

\begin{equation}
    (X^{(1)}/n)^{1/4} \sim \mathcal{N}(1 - \frac{3}{16n} - \frac{7}{512n^2} + \frac{231}{8192n^3}, \frac{1}{8n} + \frac{3}{128n^2} - \frac{23}{1024n^3})
\end{equation}

We can separately define a unit Gaussian variate, $Z^{(1)} \sim \mathcal{N}(0, 1)$ such that

\begin{equation}
    Z^{(1)} = ((X^{(1)}/n)^{1/4} - (1 - \frac{3}{16n} - \frac{7}{512n^2} + \frac{231}{8192n^3})) \cdot \frac{1}{\sqrt{\frac{1}{8n} + \frac{3}{128n^2} - \frac{23}{1024n^3}}}
\end{equation}

Notice this is just the standard reparameterization trick reversed \citep{rezende2014stochastic}.

Independently, we can define a second $\chi^2_n$ variate, $X^{(2)}$ and unit Gaussian variate $Z^{(2)}$ in the same manner.

\begin{equation}
    Z^{(2)} = ((X^{(2)}/n)^{1/4} - (1 - \frac{3}{16n} - \frac{7}{512n^2} + \frac{231}{8192n^3})) \cdot \frac{1}{\sqrt{\frac{1}{8n} + \frac{3}{128n^2} - \frac{23}{1024n^3}}}
\end{equation}

As each $Z$ is distributed as $\mathcal{N}(0, 1)$, the inverse CDF transform amounts to:

\begin{equation}
    Z^{(2)} = -Z^{(1)}
\end{equation}

Expanding each $Z$, we can derive a closed form solution:

\begin{align}
    ((X^{(2)}/n)^{1/4} - (1 - \frac{3}{16n} - \frac{7}{512n^2} + \frac{231}{8192n^3})) &= ((X^{(1)}/n)^{1/4} - (1 - \frac{3}{16n} - \frac{7}{512n^2} + \frac{231}{8192n^3})) \\
    (X^{(2)}/n)^{1/4} &= 2(1 - \frac{3}{16n} - \frac{7}{512n^2} + \frac{231}{8192n^3})) - (X^{(1)}/n)^{1/4} \\
    X^{(2)} &= n[2(1 - \frac{3}{16n} - \frac{7}{512n^2} + \frac{231}{8192n^3})) - (X^{(1)}/n)^{1/4}]^4
\end{align}

This is the approximation we use in the main text. Coincidentally, \citep{wilson1931distribution} present a similar approximation but as a third root that is more popular. In the main text, we noted that we could not use this as it led negative antithetic variances. To see why, we first write their approximation:

\begin{equation}
    (X^{(1)}/n)^{1/3} \sim \mathcal{N}(1 - \frac{2}{9n}, \frac{2}{9n})
\end{equation}

Following a similar derivation, we end with the following antithetic Wilson-Hilferty equation:

\begin{equation}
    X^{(2)} = n[2(1 - \frac{2}{9n}) - (x^{(1)}/n)^{1/3}]^3
\end{equation}

The issue lies in the cube root. If $(x^{(1)}/n)^{1/3} \geq 2(1 - \frac{2}{9n})$, then inference is ill-posed as a Normal distribution with 0 or negative variance does not exist.

\section{Cheng's Solution to the Constrained Sampling Problem}
\label{sec:methods}

In the main text, we frequently reference a second algorithm, other than \citep{marsaglia1980c69} to solve the constrained sampling problem. Here we walk through the derivation of \citep{cheng1984generation,pullin1979generation} (which we present results for in the main text):

We first review a few useful characteristics of Gamma variables, then review an important transformation with desirable properties, and finally apply it to draw representative samples from a Gaussian distribution.

\subsection{Invariance of Scaling Gamma Variates}
\label{sec:gamma}

We wish to show that Gamma random variables are closed under scaling by a constant and under normalization by independent Gamma variates.

\begin{lemma} If $x \sim \textup{Gamma}(\mu, \alpha)$ where $\mu > 0$ represents shape and $\alpha > 0$ represents rate, and $y = cx$ for some constant $c \in \mathbb{R}^{+}$, $y \sim \textup{Gamma}(\mu, \frac{\alpha}{c})$.
\label{lemma:mul_gamma}
\end{lemma}
\begin{proof}
In generality, let the chain rule be $f_y(y) = F_x(g^{-1}(y))|\frac{dx}{dy}|$ where $f$ is the cumulative distribution function for a random variable. Applying this to a Gamma: $F_y(y) = \frac{\alpha^{\mu}(y/k)^{\mu-1} \exp^{-\alpha y/k}}{\Gamma(\mu)} = \frac{(\alpha /k)^{\mu}Y^{\mu-1}\exp^{-y\cdot \alpha/k}}{\Gamma(\mu)} = \textup{Gamma}(\mu, \frac{\alpha}{k})$.
\end{proof}

\begin{lemma} Let $x_1, x_2, ..., x_k$ be $\textup{Gamma}(\mu, \alpha)$ variates and let $x_{k+1}$ be a $\textup{Gamma}(k\mu, \alpha)$ variate independent of $x_i$, for $i = 1, ..., k$. Then, $y_i = x_{k+1}(\frac{x_i}{\sum_{j=1}^{k} x_j})$ where $y_i \sim \textup{Gamma}(\mu, \alpha)$.
\label{lemma:many_gamma}
\end{lemma}

\begin{proof} See \cite{aitchison1963inverse}.
\end{proof}

\begin{lemma} If $x \sim \mathcal{N}(0, 1)$, then $x^2 \sim \chi^2_1$. Additionally, $x^2 \sim \textup{Gamma}(\frac{1}{2}, \frac{1}{2})$.
\label{lemma:normal_chi}
\end{lemma}
\begin{proof} By definition.
\end{proof}

\begin{corollary} If $x \sim \mathcal{N}(0, \sigma^2)$, then $\frac{x^2}{\sigma^2} \sim \chi^2_1$. Furthermore, we can say $x^2 \sim \sigma^2 \chi^2_1 = \sigma^2 \cdot \textup{Gamma}(\frac{1}{2}, \frac{1}{2}) = \textup{Gamma}(\frac{1}{2}, \frac{1}{2\sigma^2})$.
\label{corollary:normal_gamma}
\end{corollary}

\begin{proof} Direct application of Lemma~\ref{lemma:mul_gamma},~\ref{lemma:normal_chi}.
\end{proof}

\subsection{Helmert's Transformation}
\label{sec:helmert}

Given a random sample of size $k$ from any Gaussian distribution, Helmert's transformation \citep{helmert1875berechnung,pegoraro2012transformation} allows us to get $k-1$ new i.i.d. samples normally distributed with zero mean and the same variance as the original distribution:

Let $x_1, ..., x_k \sim \mathcal{N}(\mu, \sigma^2)$ be $k$ i.i.d. samples. We define the Helmert transformed variables, $y_2, ..., y_k$ as:

\begin{equation}
    y_j = \frac{\sum_{i=j}^{k} x_i - (k + 1 - j) x_{j-1}}{[(k + 1 - j)(k + 2 - j)]^{1/2}}
\label{eqn:helmert}
\end{equation}

for $j = 2, ..., k$. Helmert's transformation guarantees the following for new samples:

\begin{prop}
    $y_2, ..., y_k$ are independently distributed according to $\mathcal{N}(0, \sigma^2)$ such that $\sum_{i=2}^{k} y_i^2 = \sum_{i=1}^{k}(x_i - \bar{x})^2$ where $\bar{x} = \frac{1}{k}\sum_{i=1}^{k} x_i$.
\label{prop:helmert}
\end{prop}
\begin{proof} See \cite{helmert1875berechnung} or \cite{kruskal1946helmert}.
\end{proof}

Critically, Prop.~\ref{prop:helmert} also informs us that (1) the sample variance of $y_2, ..., y_k$ is equal to the sample variance of $x_1, ..., x_k$, and (2) $y_i, i=2,...,k$ can be chosen independently of $\bar{x}$. These properties will be important in the next subsection.

\subsection{Choosing Representative Samples}
\label{sec:choose}



We are tasked with the following problem: we wish to generate $k$ i.i.d. samples $x_1, ..., x_k \sim \mathcal{N}(\mu, \sigma^2)$ subject to the following constraints:

\begin{align}
    \frac{1}{k}\sum_{i=1}^{k} x_i = \bar{x} &= \eta \label{eqn:constraint1}\\
    \frac{1}{k}\sum_{i=1}^{k}(x_i - \bar{x})^2 = s^2 &= \delta^2 \label{eqn:constraint2}
\end{align}

where by definition $\bar{x} \sim \mathcal{N}(\mu, \frac{\sigma^2}{k})$ and $(k-1)s^2 / \sigma^2 \sim \chi^2_{k-1}$. We assume that $\eta$ and $(k-1)\delta^2/\sigma^2$ are particular values drawn from these respective sample distributions. In other words, given all the possible sets of $k$ samples, we wish to choose a single set such that the sample moments match a particular value, $\bar{x} = \eta$ and $s^2 = \delta^2$. Note that this is \textit{not} the same as choosing any $\eta \in \mathbb{R}$ and $\delta^2 \in \mathbb{R}$.



This problem is difficult as the number of sets of $k$ samples that do not satisfy Constraints~\ref{eqn:constraint1} and \ref{eqn:constraint2} is much larger than the number of sets that do. Thus, randomly choosing $k$ samples will not work. Furthermore, preserving that the samples are i.i.d. makes this much more difficult as we cannot rely on common methods like sampling without replacement, rejecting samples, etc.

To tackle this, \cite{pullin1979generation} used the two following observations: (1) we can handle Constraint~\ref{eqn:constraint1} independently, and (2) as a linear transformation, Helmert is invertible. First, we investigate observation 1:

Helmert's transformations allows us to untie Constraint~\ref{eqn:constraint1} from Constraint~\ref{eqn:constraint2} as $y_2, ..., y_k$ are not dependent on $\mu$ or $\eta$. Suppose we instantiate a new variable, $y_1$ \citep{kendall1946advanced} such that

\begin{equation}
    \eta = \mu + y_1 / \sqrt{k}
\label{eqn:y1}
\end{equation}

As $\eta \sim \mathcal{N}(\mu, \frac{\sigma^2}{k})$, $y_1$ is then distributed as $\mathcal{N}(0, \sigma^2)$ by reparameterization. This means that we can deterministically choose a value for $y_1$ given $\mu$ and $\eta$ to satisfy Constraint~\ref{eqn:constraint1}.

Next, satisfying Constraint~\ref{eqn:constraint2} amounts to sampling $y_2, ..., y_k$ according to Prop.~\ref{prop:helmert}.
To do this, we follow \citep{cheng1984generation} and use the Gamma properties we introduced in Section~\ref{sec:gamma}:

First, we draw $k - 1$ independent samples from $z_2, ..., z_k \sim \mathcal{N}(0, 1)$. Compute $c_2, ..., c_k$ where $c_i = (z_i * \sigma)^2$. \cite{cheng1984generation} defines $y_i, i=2, ..., k$ such that

\begin{equation}
    y_i^2 = \frac{(k - 1)\delta^2 \cdot c_i}{\sum_{j=2}^{k} c_j}
\end{equation}

By design, $\sum_i y_i^2 = (k-1)\delta^2$, as desired by Prop.~\ref{prop:helmert}. Furthermore, as $c_i \sim \textup{Gamma}(\frac{1}{2}, \frac{1}{2\sigma^2})$ and $(k-1)\delta^2 \sim \textup{Gamma}(\frac{k-1}{2}, \frac{1}{2\sigma^2})$, Lemma~\ref{lemma:many_gamma} tells us that $y_i^2$ are also distributed as $\textup{Gamma}(\frac{1}{2}, \frac{1}{2\sigma^2})$, which crucially guarantees $y_i \sim \mathcal{N}(0, \sigma^2)$ by Corollary~\ref{corollary:normal_gamma}. For $i = 2, ..., k$, we do the following:

\begin{equation}
    y_i = b'_i \cdot \sqrt{y_i^2}
\end{equation}

where $b_i = \textup{Bern}(0.5)$ and $b'_i = 2b_i- 1$ i.e. we randomly attach a sign to $y_i$. Finally, now that we know how to generate $y_1, ..., y_k$, we use \cite{pullin1979generation}'s second observation to transform $y_i$ back to $x_i$:

Precisely, the inverse of Eqn.~\ref{eqn:helmert} (Helmert) is the following:

\begin{align}
    x_1 &= \frac{1}{k}(k\eta - \sqrt{k(k-1)}y_2) \\
    x_j &= x_{j-1} + \frac{(k+2-j)^{1/2}y_j - (k-j)^{1/2}y_{j+1}}{(k + 1 - j)^{1/2}}
\label{eqn:inverse_helmert}
\end{align}

for $j = 2, ..., k$. By the ``inverse" of Prop.~\ref{prop:helmert}, Eqn.~\ref{eqn:inverse_helmert} will transform $y_1, ..., y_k$ to samples $x_1, ..., x_k \sim \mathcal{N}(0, \sigma^2)$ such that the sample mean is $\eta - \mu$ and the sample variance is $\delta^2$. Lastly, adding $x_i = x_i + \mu$ for $i=1, ..., k$ ensures samples from the correct marginal distribution along with the correct sample moments.


We refer to this procedure as \textsc{ChengSample}, detailed in Alg.~\ref{algo:match_moments}. We summarize the properties of \textsc{ChengSample} in the following proposition.

\begin{prop}
Given $k - 1$ i.i.d samples $z_1, ..., z_{k-1} \sim \mathcal{N}(0, 1)$; $k-1$ i.i.d samples $b_1, ..., b_{k-1} \sim \textup{Bern}(0.5)$; population moments from a Gaussian distribution $\mu \in \mathbb{R}, \sigma^2 \in \mathbb{R}$; and desired sample moments $\eta$, $\delta^2$ such that $\eta \sim \mathcal{N}(\mu, \frac{\sigma^2}{k})$ and $(k-1)\delta^2/\sigma^2 \sim \chi^2_{k-1}$,
generated samples $x_1, ..., x_k$ from $\textsc{ChengSample}([z_1, ..., z_{k-1}], [b_1, ..., b_{k-1}], \mu, \sigma, \eta, \delta, k)$ are (1) i.i.d., (2) marginally distributed as $\mathcal{N}(\mu, \sigma^2)$, and (3) have a sample mean of $\eta$ and a sample variance of $\delta^2$.
\label{lemma:important}
\end{prop}


\begin{algorithm}[h!]
\SetAlgoLined
\caption{\textsc{ChengSample}}
\KwData{i.i.d. samples $z_1, ..., z_{k-1} \sim \mathcal{N}(0, 1)$; i.i.d. samples $b_1, ..., b_{k-1} \sim \textup{Bern}(0.5)$; Population mean $\mu$ and variance $\sigma^2$; Desired sample mean $\eta$ and variance $\delta^2$; Number of samples $k \in \mathbb{N}$.} 
\KwResult{A set of $k$ samples $x_1, x_2, ..., x_k$ marginally distributed as $\mathcal{N}(\mu, \sigma^2)$ with sample mean $\eta$ and sample variance $\delta^2$.}

$c_i = (z_{i-1} \cdot \sigma)^2$ for $i = 2,...,k$\;
$a = (k  - 1)\delta^2 / \sum_i c_i$\;
$y_i^2 = a \cdot c_i$ for $i = 2,...,k$\;
$y_i = (2b_{i-1}-1) \cdot \sqrt{y^2_i}$ for $i = 2,...,k$\;
$y_1 = \sqrt{k}(\eta - \mu)$\;
$\alpha_k = k^{-\frac{1}{2}}$\;
$\alpha_j = (j(j+1))^{-\frac{1}{2}}$ for $j = 2,...,k$\;
$s_k =\alpha_k^{-1}y_k$\;
\For{$j\gets k$ \KwTo $2$}{
    $x_j = (s_j - \alpha^{-1}_{j-1} y_{j-1})/j$\;
    $s_{j-1} = s_j - x_j$\;
}
$x_1 = s_1$\;
$x_i = x_i + \mu$ for $i = 1,...,k$\;
Return $x_1, ..., x_k$\;
\label{algo:match_moments}
\end{algorithm}

As a final note, we chose to use Marsaglia's solution instead of Cheng's as the former as a nice geometric interpretation and requires half as many random draws (no Bernoulli variables needed in Marsaglia's algorithm).

\section{Miscellaneous}

In the \textsc{AntitheticSample} proposition in the main text, we use the fact that the average of two unbiased estimators is an unbiased estimator. We provide the proof here.

\begin{lemma} A linear combination of two unbiased estimators is unbiased.
\label{lemma:sum}
\end{lemma}
\begin{proof} Let $e_1$ and $e_2$ denote two unbiased estimators that $\mathbb{E}[e_1] = \mathbb{E}[e_2] = \theta$ for some underlying parameter $\theta$. Define a third estimator $e_3 = k_1e_1 + k_2e_2$ where $k_1, k_2 \in \mathbb{R}$. We note that $\mathbb{E}[e_3] = k_1\mathbb{E}[e_1] + k_2\mathbb{E}[e_2]$ = $(k_1 + k_2)\theta$. Thus $e_3$ is unbiased if $k_1 + k_2 = 1$.
\end{proof}

\end{document}